%% file: main.tex
\documentclass{WileyMSP-template}
\input{packages}
\begin{document}

\pagestyle{fancy}
\rhead{\vspace{0.3cm}}

\twocolumn[
  \begin{@twocolumnfalse}
    \title{PUMA: Deep Metric Imitation Learning for Stable \\Motion Primitives}
    \maketitle

    \author{Rodrigo P\'{e}rez-Dattari$^1$,}
    \author{Cosimo Della Santina$^1$,}
    \author{and Jens Kober$^1$}
    
    \begin{affiliations}
        $^1$Department of Cognitive Robotics, Delft University of Technology, The Netherlands\\
    \end{affiliations}
    
  \end{@twocolumnfalse}
]





\begin{abstract}
\textbf{Abstract---} Imitation Learning (IL) is a powerful technique for intuitive robotic programming. However, ensuring the reliability of learned behaviors remains a challenge. In the context of reaching motions, a robot should consistently reach its goal, regardless of its initial conditions. To meet this requirement, IL methods often employ specialized function approximators that guarantee this property by construction. Although effective, these approaches come with some limitations: 1) they are typically restricted in the range of motions they can model, resulting in suboptimal IL capabilities, and 2) they require explicit extensions to account for the geometry of motions that consider orientations.
To address these challenges, we introduce a novel stability loss function, drawing inspiration from the triplet loss used in the deep metric learning literature. This loss does not constrain the DNN's architecture and enables learning policies that yield accurate results. Furthermore, it is not restricted to a specific state space geometry; therefore, it can easily incorporate the geometry of the robot's state space. We provide a proof of the stability properties induced by this loss and empirically validate our method in various settings. These settings include Euclidean and non-Euclidean state spaces, as well as first-order and second-order motions, both in simulation and with real robots. More details about the experimental results can be found in: \url{https://youtu.be/ZWKLGntCI6w}.
\end{abstract}


\input{sections/01_introduction}

\input{sections/02_related_work}
\input{sections/03_preliminaries}
\input{sections/04_methodology}

\input{sections/05_experiments}
\input{sections/06_conclusion}

\section*{\large Acknowledgements}
The authors would like to thank Xin Wang, Thomas Versmissen, Bastiaan Vroegindeweij, Rekha Raja, Gert Kootstra, and Eldert van Henten of the Agricultural Biosystems Engineering Group at Wageningen University and Research. Their assistance during the real-world experiments of this work, conducted at their facilities, was invaluable.

\section*{\large Funding}
This research is funded by the Netherlands Organization for Scientific Research project Cognitive Robots for Flexible Agro-Food Technology, grant P17-01, and by the National Growth Fund program NXTGEN Hightech.

\input{sections/appendix/appendix}

\medskip

%
\bibliographystyle{IEEEtran}  
\bibliography{IEEEexample}  

\end{document}

%% file: packages.tex
\usepackage{moreverb,url}

\usepackage[colorlinks,bookmarksopen,bookmarksnumbered,citecolor=black,urlcolor=black]{hyperref}

\newcommand\BibTeX{{\rmfamily B\kern-.05em \textsc{i\kern-.025em b}\kern-.08em
T\kern-.1667em\lower.7ex\hbox{E}\kern-.125emX}}
\usepackage[numbers]{natbib}
\usepackage{titlesec}
\usepackage{lipsum}
\usepackage{amsbsy}
\usepackage{amssymb}
\usepackage{amsmath}
\usepackage{color}
\usepackage[colorinlistoftodos,prependcaption,textsize=tiny]{todonotes}
\usepackage{subfig}
\usepackage{graphicx}
\usepackage{svg}
\usepackage{hyperref}
\usepackage{array}
\usepackage{makecell}
\usepackage{pifont}

\usepackage{enumitem}
\usepackage[export]{adjustbox}
\usepackage[flushleft]{threeparttable}
\usepackage{booktabs}
\usepackage{mathtools}
\usepackage{float}
\usepackage{mathrsfs}
\usepackage{microtype}
\usepackage{ragged2e}
\justifying

\PassOptionsToPackage{hyphens}{url}\usepackage{hyperref}
\hypersetup{
     colorlinks=true,
     linkcolor=black,
     filecolor=black,      
     }
\usepackage{mathtools}
\mathtoolsset{showonlyrefs} 
\usepackage{amsthm}
\newtheorem{theorem}{Theorem}
\newtheorem{proposition}{Proposition}
\newtheorem{lemma}{Lemma}
\newtheorem{corollary}{Corollary}
\theoremstyle{definition}
\newtheorem{definition}{Definition}
\theoremstyle{definition}

\let\citet\cite
\let\citep\cite

\DeclareMathOperator*{\argmin}{arg\,min}

\usepackage{titlesec}
\renewcommand{\theparagraph}{\alph{paragraph})}
\titleformat{\paragraph}[hang]{\normalfont\normalsize\itshape}{\hspace{1em}\theparagraph}{0.5em}{}
\titlespacing*{\paragraph}{0pt}{\baselineskip}{0em}

\let\oldsubsubsection\subsubsection

\renewcommand{\subsubsection}[1]{\oldsubsubsection{#1}\mbox{}\par}

\newcommand{\cmark}{\ding{51}}%
\newcommand{\xmark}{\ding{55}}%

\let\oldmaketitle\maketitle 
\renewcommand{\maketitle}{\oldmaketitle\setcounter{footnote}{1}}

\newcommand\lword[1]{\leavevmode\nobreak\hskip0pt plus\linewidth\penalty50\hskip0pt plus-\linewidth\nobreak{#1}}

%% file: sections/01_introduction.tex
\section{Introduction}
\begin{figure}[t]
    \centering
    \includegraphics[width=0.9\columnwidth]{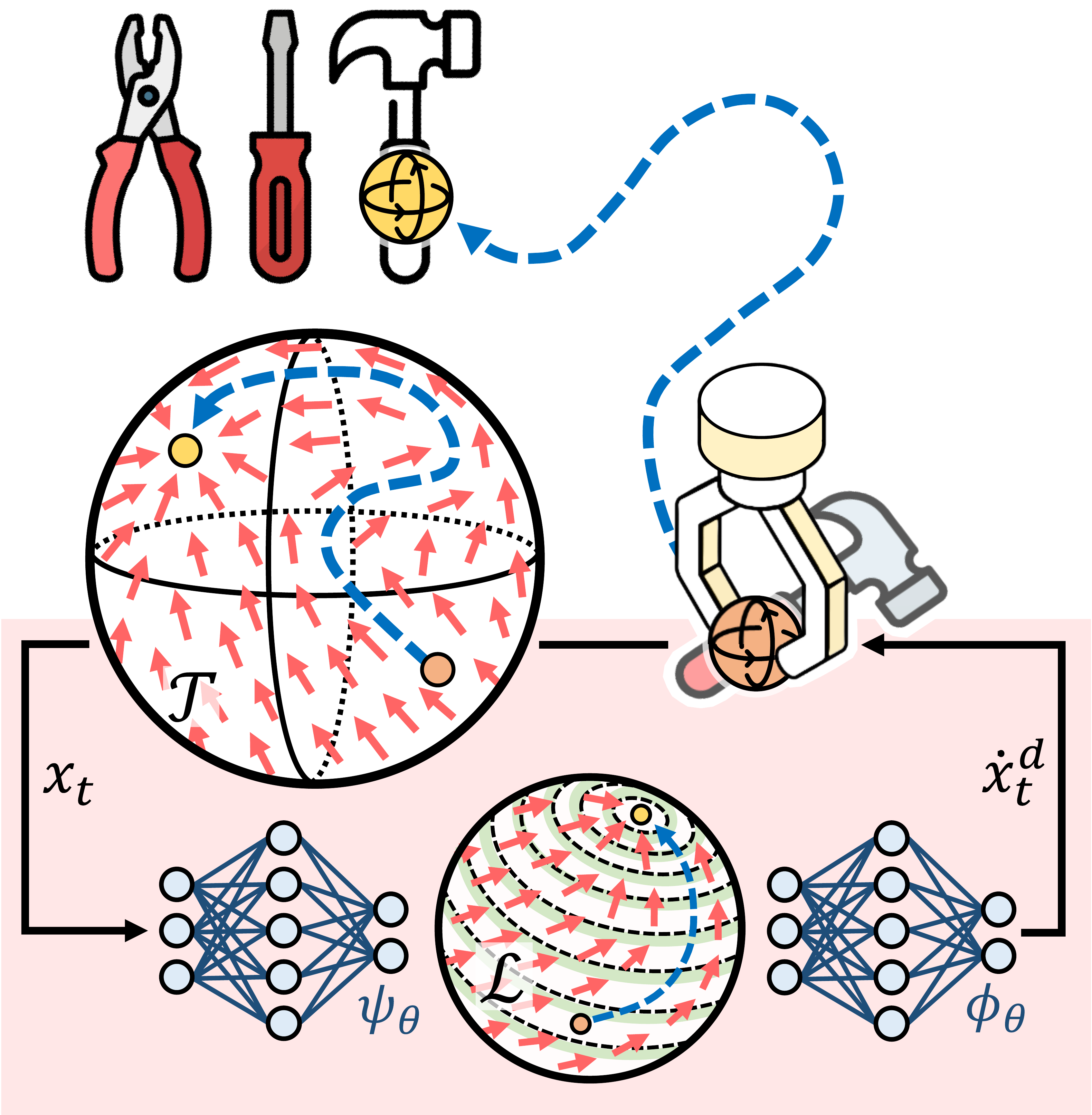}
    \caption{Motion learned using the proposed framework. The blue trajectory in the task space $\mathcal{T}$ demonstrates the evolution of the robot's end effector state $x_{t}$ when represented in a spherical manifold. The evolution of this trajectory is governed by the dynamical system $\dot{x}_{t}=\phi_{\theta}(\psi_{\theta}(x_{t}))$, depicted as a vector field of red arrows in the remaining of the space. Through Deep Metric Learning, this system is stabilized by deriving a simpler representation in the latent space $\mathcal{L}$.}
    \label{fig:cover}
\end{figure}

Imitation Learning (IL) provides a powerful framework for the intuitive programming of robotic systems. Its strength lies in its ability to leverage human-like learning methodologies, such as demonstrations and corrections, making it accessible to non-robotics experts. This attribute significantly reduces the resources needed to build robotic systems. However, the data-driven nature of these methodologies presents a challenge: providing guarantees about the learned behaviors.

In the context of reaching motions, it is crucial for the robot's motions to consistently reach the intended target, irrespective of the robot's initial conditions. Thus, modeling motions as dynamical systems proves beneficial. This approach turns the problem into a question of ensuring \emph{global asymptotic stability} (or \emph{stability}, for short) at the goal, and tools from dynamical system theory can then be applied to guarantee this property.

Numerous methods have been proposed to ensure stability in motions represented by dynamical systems. However, they often exhibit at least one of the following limitations: 1) constraining the structure of their function approximators, and/or 2) being designed with the assumption that the robot's state space is Euclidean. To elaborate:

\begin{enumerate}[leftmargin=*]
    \item \textbf{Constrained function approximators.} To ensure stability guarantees, methods often constrain the structure of their function approximators. For instance, some approaches necessitate invertibility \citep{perrin2016fast,rana2020euclideanizing,urain2020imitationflow}, while others require positive or negative definiteness \citep{khansari2011learning,khansari2014learning,lemme2014neural}. However, these methods do not enable the full exploitation of modern Deep Neural Network (DNN) architectures, as these constraints are not typically present in DNNs. This limitation hinders their broader application in more complex models, where integrating these constraints is challenging. Furthermore, inherently constraining function approximators can overly restrict the range of solutions to which they can converge, resulting in less flexible models than necessary. This leads to suboptimal IL capabilities. In our context, this problem is known as the \emph{stability versus accuracy dilemma} \cite{figueroa2022locally}.
    
    \item \textbf{Euclidean assumption.} Learned motions must integrate with the geometry of the space used to represent a robot's state. For example, the end-effector of a manipulator should always reach the intended target in both position and orientation space. However, orientations are often represented in non-Euclidean spaces, such as $SO(3)$ or $\mathcal{S}^{3}$. This is because Euclidean representations like Euler angles may not always provide a continuous description of motions that are inherently continuous\footnote{This issue stems from singularities and non-unique representations.} \citep{ude2014orientation}. Such continuity is often a requirement for motions modeled as dynamical systems. Furthermore, the generalization capabilities of function approximations are compromised by non-continuous representations. To enable proper generalization, states that are close in the real world should be represented as closely related in the function approximator's input, i.e., the state space representation.

    Nevertheless, previous methods for stable motion generation have been initially designed under the assumption that the state space is Euclidean \citep{khansari2011learning,rana2020euclideanizing,urain2020imitationflow,perez2023stable}. Consequently, some have later been explicitly adapted to account for the geometry of motions that consider orientations \citep{zhang2022learning,urain2022learning,saveriano2023learning}, resulting in rather convoluted learning frameworks. 
\end{enumerate}

In this work, we present a DNN framework capable of learning accurate, stable motions in state spaces with arbitrary geometries without constraining the DNN's architecture. To accomplish this, we introduce a novel loss function that repurposes the \emph{triplet loss}, commonly used in deep metric learning literature \citep{schroff2015facenet,kaya2019deep}. We prove that this loss imposes conditions on the DNN's latent space that enforce the learned dynamical system to have a globally asymptotically stable equilibrium at the motion's goal state (see Fig.~\ref{fig:cover}). To account for the geometry of the state space, it is sufficient to consider its corresponding metric during the computation of the loss, the choice of which depends on the specific task at hand. We validate our method in various settings, including Euclidean, non-Euclidean, first-order, and second-order motions. Additionally, real-world experiments controlling the 6-dimensional pose (x-y-z position and unit quaternion orientation) of a robot manipulator's end effector demonstrate the method's practical applicability and potential.

The rest of this paper offers a thorough discussion of our developments. Following a review of the relevant literature, we delve into foundational concepts and problem formulation. We then present the details of our methodology, provide a proof regarding the stability of the learned motion, and discuss the integration of the triplet loss function within the context of non-Euclidean state representations. We validate our approach through several experiments and conclude by considering potential directions for future research based on our findings.

%% file: sections/02_related_work.tex
\section{Related Works}\label{sec:related_works}
Three categories of works are relevant to this paper. First, there are papers that focus on learning stable motions, assuming these motions occur in Euclidean state spaces. Second, others explore learning motions in non-Euclidean state spaces, but do not consider the stability of the learned motions. Finally, a third set of papers addresses both learning motions in non-Euclidean state spaces and their stability. Importantly, 

In every case, works use either \emph{time-varying (non-autono\-mous)} or \emph{time-invariant (autonomous)} dynamical systems for motion modeling. In time-varying systems, evolution explicitly depends on time (or a phase). Conversely, time-invariant systems do not directly depend on time; instead, they rely on their time-varying input (i.e., the state of the system). The property of a system being time-invariant or not dictates the strategies we can use to ensure its stability. Thus, making a distinction between these systems is important. Notably, both types of formulations have been shown to be complementary in the context of IL \citep{calinon2011encoding,ravichandar2020recent}.

This work focuses on time-invariant dynamical systems. Consequently, although we consider both methodologies for motion learning, we delve deeper into the literature on time-invariant systems. Furthermore, while we concentrate on methods that ensure asymptotic stability, we acknowledge that there are studies imposing other conditions, such as via-point conditioning \cite{paraschos2013probabilistic, seker2019conditional}, which can be significantly relevant in some robotic contexts.

\subsection{Stability in Euclidean State Spaces}
Regarding time-varying dynamical systems, a seminal work in IL that addresses the problem of learning stable motions introduces Dynamical Movement Primitives (DMPs) \citep{ijspeert2013dynamical}. DMPs take advantage of the time-dependency (via the phase of the \emph{canonical system}) of the dynamical system to make it evolve into a simple and well-understood system as time goes to infinity. The simple system is designed to be stable by construction. As a consequence, the stability of the learned motions can be guaranteed. This concept has been extended in multiple ways, for instance through probabilistic formulations \citep{li2023prodmp,amor2014interaction}, or adapted to the context of DNNs \citep{pervez2017learning,ridge2020training,bahl2020neural}.

Conversely, IL approaches based on time-invariant dynamical systems often constrain the function approximator used to model the dynamical systems. Such constraints ensure stability by construction. One seminal work introduces the Stable Estimator of Dynamical Systems (SEDS) \citep{khansari2011learning}. This approach imposes constraints on the structure of a Gaussian Mixture Regression (GMR) such that the conditions for Lyapunov stability are always met. Multiple extensions of SEDS have been proposed, for instance by using physically-consistent priors \citep{figueroa2018physically}, contraction theory \citep{ravichandar2017learning} or diffeomorphisms \citep{neumann2015learning}. 

Other works propose explicitly learning Lyapunov functions that are consistent with the demonstrations. These functions are then used to correct the transitions learned by the dynamical systems, ensuring that they are always stable according to the learned Lyapunov functions \citep{khansari2014learning,lemme2014neural,duan2017fast}. Moreover, some papers have employed concepts such as contraction metrics \citep{ravichandar2015learning,blocher2017learning} and diffeomorphisms \citep{perrin2016fast,rana2020euclideanizing,urain2020imitationflow, zhi2022diffeomorphic} to impose stability, in the sense of Lyapunov, in time-invariant dynamical systems.

Understandably, all of these methods constrain some part of their learning framework to ensure stability. From one perspective, this is advantageous as it guarantees stability. However, in many cases, this comes at the expense of reduced accuracy in the learned motions. Notably, some recent methods have managed to mitigate this loss in accuracy \citep{rana2020euclideanizing, urain2020imitationflow}. Nevertheless, they are still limited in terms of the family of models that can be used with these frameworks, which harms their scalability. 

In previous work \citep{perez2023stable}, we addressed this issue by employing tools from the deep metric learning literature \citep{kaya2019deep}. There, we introduced CONDOR, which uses a \emph{contrastive loss} to enforce stability in learned motions through the optimization process of a DNN. This approach proved effective in learning stable, accurate, and scalable motions from human demonstrations. However, this method presents two important limitations. First, it requires the design of a stable and well-understood system for computing the contrastive loss (referred to as $f^{\mathcal{L}}$ in Sec. \ref{sec:stability_conditions}), which can significantly impact learning performance. Second, similar to the other methods described in this subsection, CONDOR is limited to operating within Euclidean state spaces.

In this work, we introduce a novel deep metric learning loss that ensures stability while addressing the limitations of \citet{perez2023stable}, specifically, the need for the function $f^{\mathcal{L}}$ and the inability to learn motions in non-Euclidean manifolds. Moreover, this loss requires weaker conditions for imposing stability, leading to, for instance, more robust performance when learning second-order dynamical systems.

\subsection{Stability in Non-Euclidean State Spaces}
As noted previously, it is crucial to consider the geometry of our state spaces when learning motions requiring pose control. Consequently, many studies have adapted methods that originally assumed data from Euclidean spaces to work with data from non-Euclidean manifolds, such as $SO(3)$ and $\mathcal{S}^{3}$. In the context of IL, pioneering approaches utilized time-varying dynamical systems (extensions of DMPs) to incorporate the geometry of orientation representations into their models \citep{pastor2011online,ude2014orientation,koutras2020correct}. Because DMPs inherently ensure stability, these methods are stable as well. 

In contrast, although several geometry-aware IL approaches based on time-invariant dynamical systems have been introduced, such as Gaussian process regressions \citep{lang2017computationally,arduengo2023gaussian}, GMRs \citep{silverio2015learning,havoutis2019learning,zeestraten2017approach,kim2017gaussian}, and kernelized movement primitives \citep{huang2020toward,abu2021probabilistic}, such methods are not inherently stable. As a result, in these cases, models need to be explicitly endowed with stability properties. This limitation has been addressed in recent works, where \citet{ravichandar2019learning} uses GMRs and employs contraction theory to ensure stability considering the robot's orientation geometry. Furthermore, recent methods have extended the use of diffeomorphism-based techniques for stability to generate motions in non-Euclidean manifolds as well \citep{zhang2022learning,urain2022learning,saveriano2023learning}.

These diffeomorphism-based methods are of particular interest to our work, as our method is grounded in similar concepts for achieving stability. In both approaches, we transfer the stability properties of a simple system in the latent space of a DNN to the dynamical system that models motion in task space. However, unlike these methods, our approach learns this property, whereas the other works constrain the DNN structure to ensure its satisfaction. Moreover, our method is not constrained to diffeomorphic solutions for transferring the stability properties of the simple system to task space. Lastly, our method allows us to seamlessly incorporate the geometrical aspects of the robot's state space into the learned model, as this integration is factored into the DNN's optimization process.

%% file: sections/03_preliminaries.tex
\section{Preliminaries}
\label{sec:preliminaries}
\subsection{Dynamical Systems for Reaching Tasks}
In this work, we model motions as nonlinear time-invariant dynamical systems represented by 
\begin{equation}
\label{eq:ds}
    \dot{x}_{t} = f(x_{t}),  
\end{equation}
where $x_{t} \in \mathcal{T}$, with $\mathcal{T} \subseteq \mathbb{R}^{n}$ being the task space, is the robot's state, and ${f: \mathcal{T} \to \mathbb{R}^{n}}$ is a differentiable function. Additionally, the subscript $t$ indicates the time instance to which the state corresponds.
Since we are interested in solving reaching tasks, we aim to construct systems with a globally asymptotically stable equilibrium at a goal state $x_{\mathrm{g}}$. This implies that
\begin{equation}\label{eq:stability}
    \lim_{t \rightarrow \infty} ||x_{\mathrm{g}} -  x_{t}||=0,\hspace{0.5cm} \forall x_{t} \in \mathcal{T}.
\end{equation}

\subsection{Problem Formulation}
We consider a robot learning a reaching motion in the space $\mathcal{T}$ towards the goal state $x_{\mathrm{g}} \in \mathcal{T}$. Based on a set of demonstrations $\mathcal{D}$, the robot is expected to imitate the behavior shown in the demonstrations while always reaching $x_{\mathrm{g}}$, regardless of its initial state. The dataset $\mathcal{D}$ contains $N$ trajectories $\tau$. These trajectories show the evolution of a dynamical system's state $x_{t}$ when starting from some initial condition $x_{0}$.

We assume that the demonstrations are drawn from an \emph{optimal} distribution over trajectories, denoted as $p^*(\tau)$, that adheres to the \emph{optimal} dynamical system $f^*$. Here, ``optimal" refers to the behavior that the demonstrator deems best. 

The robot's motion follows the parametrized system $f^{\mathcal{T}}_{\theta}$, inducing the distribution $p_{\theta}(\tau)$, where $\theta$ is a parameter vector. Then, the objective is to find the vector $\theta^{*}$ that minimizes the difference, expressed as the (forward) Kullback-Leibler divergence, between $p_{\theta}(\tau)$ and $p^*(\tau)$, while ensuring that $x_{\mathrm{g}}$ is a globally asymptotically stable equilibrium:
\begin{subequations}\label{eqn:problem_formulation}
\begin{align}
\theta^* = \argmin_{\theta \in \Theta} \quad & D_{\text{KL}}\left(p^{*}(\tau) ||  p_{\theta}(\tau)\right) \label{eq:divergence}\\ 
\text { s.t. } &\lim_{t \rightarrow \infty} d(x_{\mathrm{g}}, x_{t})=0, \forall x_{t} \in \mathcal{T}. 
\end{align}
\end{subequations}
Here, $\Theta$ is the parameter space of a DNN, and $d(\cdot,\cdot)$ is a distance function.

\subsection{Stability Conditions}
\label{sec:stability_conditions}

In \citet{perez2023stable}, the \emph{stability conditions} are formulated to enforce stability in $f^{\mathcal{T}}_{\theta}$.
This is achieved by ensuring that $f^{\mathcal{T}}_{\theta}$ inherits the stability properties of a simple and stable system, referred to as $f^{\mathcal{L}}$. Thus, if $f^{\mathcal{L}}$ is asymptotically stable, $f^{\mathcal{T}}_{\theta}$ will also be asymptotically stable. 
To accomplish this, the system $f^{\mathcal{L}}$ is designed to define the evolution of a state $y_{t}$ with an initial condition derived from mapping $x_{0}$ to the output of a hidden layer in the DNN that parameterizes $f^{\mathcal{T}}_{\theta}$. As a result, the dynamical system $f^{\mathcal{T}}_{\theta}$ is expressed as a composition of two functions, $\psi_{\theta}$ and $\phi_{\theta}$,
\begin{equation}
    \dot{x}_{t}=f^{\mathcal{T}}_{\theta}(x_{t}) = \phi_{\theta}(\psi_{\theta}(x_{t})).
\end{equation}
Here, $f^{\mathcal{T}}_{\theta}$ is a standard DNN with $L$ layers. $\psi_{\theta}$ denotes layers $1,...,l$, and $\phi_{\theta}$ layers $l+1,...,L$. We define the output of layer $l$ as the latent space $\mathcal{L} \subset \mathbb{R}^{m}$, which is where $y_{t}$ resides, i.e., $y_{t} \in \mathcal{L}$. Note that $\mathcal{L} \subset \mathbb{R}^{m}$ implies the dimensionality of the vectors used to represent $\mathcal{T}$ and $\mathcal{L}$ \textbf{does not need to be the same}. Moreover, for simplicity, although we use the same $\theta$ notation for both $\psi_{\theta}$ and $\phi_{\theta}$, each symbol actually refers to a different subset of parameters within $\theta$. These subsets together form the full parameter set in $f^{\mathcal{T}}_{\theta}$.

Additionally, we introduce a third dynamical system. This system denotes the evolution in $\mathcal{L}$ of the states visited by $f^{\mathcal{T}}_{\theta}$ when mapped using $\psi_{\theta}$, which yields the relationship
\begin{equation}
    \dot{y}_{t} = f_{\theta}^{\mathcal{T}\to\mathcal{L}}(x_{t}) = \frac{\partial \psi_{\theta}(x_{t})}{\partial t}.
\end{equation}
Fig.~\ref{fig:condor_structure} provides an example of the introduced dynamical systems.

Then, the stability conditions of \citet{perez2023stable} can be written as:
\begin{theorem}[Stability conditions: v1]
\label{theo:condor_stability}
Let $f_{\theta}^{\mathcal{T}}$, $f_{\theta}^{\mathcal{T} \to \mathcal{L}}$ and $f^{\mathcal{L}}$ be the introduced dynamical systems. Then, in the region $\mathcal{T}$, $x_{\mathrm{g}}$ is a globally asymptotically stable equilibrium of $f_{\theta}^{\mathcal{T}}$ if, $\forall x_{t} \in \mathcal{T}$,:
\begin{enumerate}[font=\itshape]
    \item $f_{\theta}^{\mathcal{T}\to\mathcal{L}}(x_{t})=f^{\mathcal{L}}(y_{t})$,
    \item $\psi_{\theta}(x_{t})=y_{\mathrm{g}} \Rightarrow x_{t}=x_{\mathrm{g}}$.
\end{enumerate}
\end{theorem}
These conditions imply that if the system $f_{\theta}^{\mathcal{T}\to\mathcal{L}}$ behaves like $f^{\mathcal{L}}$, and any point other than $x_{\mathrm{g}}$ is excluded from mapping to the latent goal $y_{\mathrm{g}}$, then $f_{\theta}^{\mathcal{T}}$ is stable. Note that the latent goal is defined by the mapping $\psi_{\theta}(x_{\mathrm{g}}) = y_{\mathrm{g}}$.

In this work, we reformulate these conditions and propose a novel way to optimize them. This adaptation endows the learning framework with increased flexibility, enabling it to tackle a wider array of problems (e.g., non-Euclidean state spaces) and achieve improved performance.

\begin{figure}[t]
    \centering
    \includegraphics[width=0.9\columnwidth]{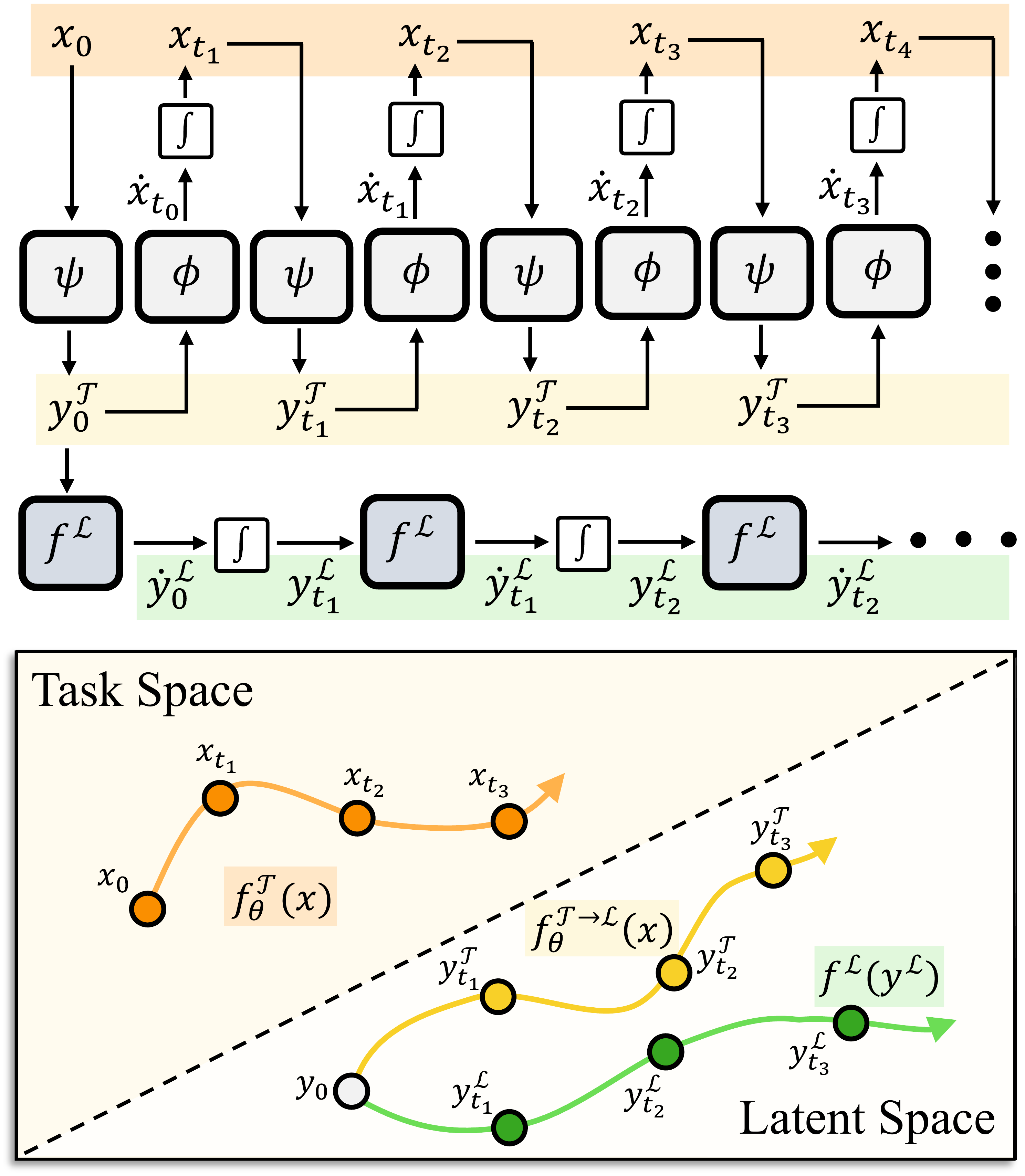}
    \caption{Example of trajectories generated by simulating the systems $f^{\mathcal{T}}_{\theta}$, $f^{\mathcal{T}\to\mathcal{L}}_{\theta}$ and $f^{\mathcal{L}}$ for different time instants. The stability conditions are not met in this case, as $f^{\mathcal{T}\to\mathcal{L}}_{\theta}$ differs from $f^{\mathcal{L}}$.}
    \label{fig:condor_structure}
\end{figure}
\subsection{Deep Metric Learning: the Triplet Loss}
\label{sec:triplet_loss}
Commonly employed in the Deep Metric Learning literature, the \emph{triplet loss} \citep{schroff2015facenet} has been utilized for learning and structuring latent state representations \citep{kaya2019deep}. Its function is to cluster similar observations together and differentiate dissimilar ones within the latent space of a DNN. In this work, however, the triplet loss is used differently. Although we continue to use this loss to impose a certain structure on the DNN's latent space, its purpose is to enforce the stability conditions, thereby ensuring that $f^{\mathcal{T}}_{\theta}$ is stable.

Let us recall the triplet loss:
\begin{equation}
    \ell_{\text{triplet}} = \max(0, m + d(a, p) - d(a, n)),
\label{eq:triplet_loss}
\end{equation}
where $m \in \mathbb{R}_{>0}$ is the margin, $a$ is the anchor sample, $p$ the positive sample, and $n$ the negative sample This loss enforces positive samples to be at least $m$ distance closer to the anchor than the negative samples. Notably, this is enough for the loss to become zero, i.e., it does not require the anchor and positive samples to have the same value.

\subsection{Stability Analysis through Comparison Functions}
In this work, we employ comparison functions, namely class-$\mathcal{KL}$ functions, to prove global asymptotic stability at the equilibrium $x_{\mathrm{g}}$. These functions are used to formulate a general approach for stability analysis in the sense of Lyapunov. As described in  \citet{khalil2002nonlinear, kellett2014compendium}, these functions are defined as follows:
\begin{definition}[class-$\mathcal{K}$ function]
A continuous function $\alpha:[0, a) \to \mathbb{R}_{\geq 0}$, for $a \in \mathbb{R}_{>0}$, is said to belong to class $\mathcal{K}$ if it is strictly increasing and $\alpha(0)=0$.
\end{definition}
\begin{definition}[class-$\mathcal{L}$ function\footnote{Note that the symbol $\mathcal{L}$ is used in two distinct contexts. While it represents the class-$\mathcal{L}$ functions, it is also used to denote the set $\mathcal{L}$ introduced in Sec. \ref{sec:stability_conditions}. In subsequent sections of the paper, the $\mathcal{L}$ referring to class-$\mathcal{L}$ functions is exclusively used in the form $\mathcal{KL}$.}]
A continuous function ${\sigma:\mathbb{R}_{\geq 0} \to \mathbb{R}_{>0}}$, is said to belong to class $\mathcal{L}$ if it is \emph{weakly decreasing}\footnote{We use this term to denote functions that either remain constant or strictly decrease within any interval of their domain.} and $\lim_{s \to \infty}\sigma(s)=0$.
\end{definition}
\begin{definition}[class-$\mathcal{KL}$ function]
\label{def:classKL}
A function $\beta: [0, a) \times \mathbb{R}_{\geq 0} \to \mathbb{R}_{\geq 0}$, for $a \in \mathbb{R}_{>0}$, is said to belong to class-$\mathcal{KL}$ if:
\begin{itemize}
    \item for each fixed $s$, the mapping $\beta(r, s)$ belongs to class-$\mathcal{K}$ with respect to r,
    \item for each fixed $r$, the mapping $\beta(r, s)$ belongs to class-$\mathcal{L}$ with respect to s.
\end{itemize}
\end{definition}

Then, we can describe global asymptotic stability in terms of class-$\mathcal{KL}$ functions as:
\begin{theorem}[Global asymptotic stability with class-$\mathcal{KL}$ functions]
\label{theo:kl_stability}
    The state $x_{\mathrm{g}}$ is a globally asymptotically stable equilibrium of \eqref{eq:ds} in $\mathcal{T}$ if there exists a class-$\mathcal{KL}$ function $\beta$ such that, $\forall t \in \mathbb{R}_{\geq 0}$ and $\forall x_{0} \in \mathcal{T}$,
    \begin{equation}
    \label{eq:beta_condition}
        ||x_{\mathrm{g}} - x_{t}|| \leq \beta(||x_{\mathrm{g}} - x_{0}||, t).
    \end{equation}
\end{theorem}
Note that this theorem seamlessly integrates the concepts of \emph{stability} and \emph{attractivity} within a single function $\beta$, which acts as an upper bound of $||x_{\mathrm{g}} - x_{t}||$. As the initial condition of the system moves further away from $x_{\mathrm{g}}$, $\beta$ correspondingly increases. Moreover, as the system evolves over time and $\beta$ decreases, it follows that the system's distance to $x_{\mathrm{g}}$ will eventually decrease.

%% file: sections/04_methodology.tex
\section{Methodology}
\label{sec:method}
We introduce the \emph{Policy via neUral Metric leArning} (PUMA) framework, which learns motion primitives from human demonstrations parametrized as the dynamical system $f^{\mathcal{T}}_{\theta}$. Furthermore, it enforces the goal of the motion $x_{\mathrm{g}}$ to be a globally asymptotically stable equilibrium of $f^{\mathcal{T}}_{\theta}$ while maintaining accuracy with respect to the demonstrations. To achieve this, we augment the Imitation Learning (IL) problem with the stability-enforcing loss $\ell_{\text{stable}}$. Hence, our framework minimizes the loss:
\begin{equation}
    \ell_{\scriptscriptstyle\mathrm{PUMA}} = \ell_{\text{IL}} + \lambda \ell_{\text{stable}},
\label{eq:puma_loss}
\end{equation}
where $\ell_{\text{IL}}$ is an imitation learning loss and $\lambda \in \mathbb{R}_{>0}$ is a weight factor.

\subsection{Behavioral Cloning}
\label{sec:behavioral_cloning}
\begin{figure}[t]
    \centering
    \includegraphics[width=0.9\columnwidth]{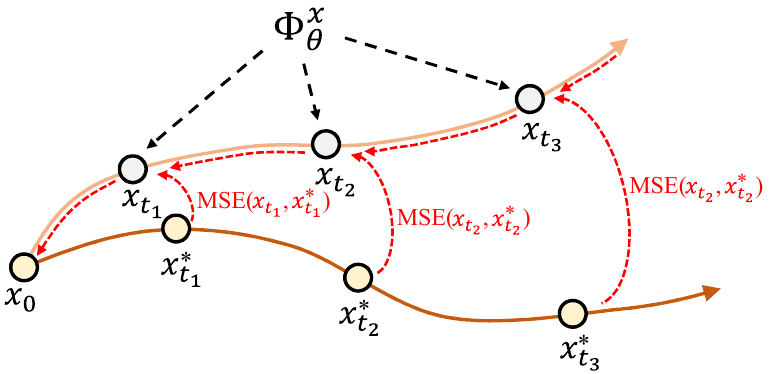}
    \caption{Illustration of the behavioral cloning loss computation. Starting from an initial condition $x_{0}$, the system $f^{\mathcal{T}}_{\theta}$ evolves to various time instants via $\Phi^{x}_{\theta}$. At each instant, the estimated state is compared with a demonstrated state. The red arrows show the gradient path used to update the DNN's weights using BPTT.}
    \label{fig:bc}
\end{figure}
To minimize $\ell_{\text{IL}}$, and thereby address~\eqref{eq:divergence}, we adopt the behavioral cloning loss used in \citet{perez2023stable}. This loss tackles~\eqref{eq:divergence} by modeling the output of the deterministic dynamical system $f^{\mathcal{T}}_{\theta}$ as the mean of a Gaussian distribution with fixed covariance. Furthermore, it mitigates \emph{covariate shift} by minimizing the multi-step error over trajectory segments using \emph{backpropagation through time} (BPTT), as depicted in Fig.~\ref{fig:bc}. 

In every training iteration, we sample a batch $\mathcal{B}^{i}$ of trajectory segments $\mathcal{H}^{i}$ from the dataset $\mathcal{D}$. These segments can start at any point within a given demonstration, with the start time defined as $t=0$. Then, by introducing the \emph{evolution function} $\Phi^{x}_{\theta}(t, x_{0}):\mathbb{R}_{\geq 0} \times \mathcal{T} \to \mathcal{T}$, which defines the value of $x_{t}$ by integrating $f^{\mathcal{T}}_{\theta}$ between $0$ and $t$, with initial condition $x_{0}$, we can construct the loss
\begin{equation}
\ell_{\text{IL}} = \sum_{\mathcal{H}^{i}\in \mathcal{B}^{i}}\sum_{(t, x^{*}_{t}) \in \mathcal{H}^{i}} ||x^{*}_{t} - \Phi^{x}_{\theta}(t, x_{0})||_{2}^{2}.
\label{eq:bc}
\end{equation}
Here, the states $x^{*}_{t}$ along the trajectory segment $\mathcal{H}^{i}$ serve as labels for the states predicted by the DNN from the initial condition $x_{0}$ using $\Phi^{x}_{\theta}(t, x_{0})$. Note that the initial condition is obtained from $\mathcal{H}^{i}$, i.e.,  $x_{0}=x^{*}_{0}$.

Since we do not have an analytical solution of $\Phi^{x}_{\theta}$, we approximate it using the \emph{forward Euler method}, i.e.,
\begin{equation}
\label{eq:forward_euler}
   \Phi^{x}_{\theta}(t, x_{0}) = \Phi^{x}_{\theta}(t', x_{0}) + f_{\theta}^{\mathcal{T}}\left(\Phi^{x}_{\theta}(t', x_{0})\right)\Delta t, 
\end{equation}
where $t' = t - \Delta t$ and $\Delta t \in \mathbb{R}_{>0}$ is the time step size. This integration starts with the initial state $x_{0}$, i.e., $\Phi^{x}_{\theta}(0, x_{0})=x_{0}$. It is important to note that the recursive nature of $\Phi^{x}_{\theta}(t, x_{0})$ necessitates the use of BPTT for the optimization of the DNN. 

\subsection{Triplet Stability Loss}
\subsubsection{Reformulating the Stability Conditions}
The stability conditions of Theorem~\ref{theo:condor_stability} involve three dynamical systems: $f^{\mathcal{T}}_{\theta}$, $f^{\mathcal{T} \to \mathcal{L}}_{\theta}$, and $f^{\mathcal{L}}$. Note, however, that its first condition, namely, $f_{\theta}^{\mathcal{T}\to\mathcal{L}}=f^{\mathcal{L}}$ $\forall x_{t} \in \mathcal{T}$, essentially states that $f_{\theta}^{\mathcal{T}\to\mathcal{L}}$ must exhibit global asymptotic stability. In other words, if the behavior of $f_{\theta}^{\mathcal{T}\to\mathcal{L}}$ is identical to that of another system, $f^{\mathcal{L}}$, for which stability is verified, then the stability of $f_{\theta}^{\mathcal{T}\to\mathcal{L}}$ is also verified. Nonetheless, if we can enforce its stability through a different method, these conditions can be more generally written as follows:

\begin{theorem}[Stability conditions: v2]
\label{theo:puma_stability}
Let $f_{\theta}^{\mathcal{T}}$ and $f_{\theta}^{\mathcal{T} \to \mathcal{L}}$ be the introduced dynamical systems. Then, in the region $\mathcal{T}$, $x_{\mathrm{g}}$ is a globally asymptotically stable equilibrium of $f_{\theta}^{\mathcal{T}}$ if, $\forall x_{t} \in \mathcal{T}$,:
\begin{enumerate}[font=\itshape]
    \item $y_{\mathrm{g}}$ is a globally asymptotically stable equilibrium of \newline$f_{\theta}^{\mathcal{T}\to\mathcal{L}}(x_{t})$,
    \item $\psi_{\theta}(x_{t})=y_{\mathrm{g}} \Rightarrow x_{t}=x_{\mathrm{g}}$.
\end{enumerate}
\end{theorem}

\subsubsection{Surrogate Stability Conditions}
\label{sec:surrogate_stability}
Theorem~\ref{theo:puma_stability} introduces stability conditions for $f^{\mathcal{T}}_{\theta}$; nevertheless, it does not specify how to enforce these conditions in the system. Therefore, we introduce the \emph{surrogate stability conditions} of Theorem~\ref{theo:puma_stability}. These conditions, when met, imply that the stability conditions of Theorem~\ref{theo:puma_stability} are also satisfied. Unlike the stability conditions, the surrogate conditions can be directly transformed into a specific loss function, $\ell_{\text{stable}}$, for optimizing the DNN to enforce their satisfaction.

To formulate the surrogate stability conditions, we note that Theorem~\ref{theo:puma_stability} can be expressed in terms of relative distances. We define the distance between any given latent state $y_t$ and the goal state $y_{\mathrm{g}}$ as $d_t = d(y_{\mathrm{g}}, y_t) = \|y_{\mathrm{g}} - y_t\|$. Then, according to Condition 1 of Theorem~\ref{theo:puma_stability}, which addresses global asymptotic stability, the value of $d_t$ should, generally speaking, decrease over time. Moreover, the second condition specifies that the value of $d_t$ should remain constant only for $y_{\mathrm{g}} = \psi_{\theta}(x_{\mathrm{g}})$.
Formally, we introduce the conditions as:

\begin{theorem}[Surrogate stability conditions]
\label{theo:puma_surrogate}
Let two dynamical systems be governed by the equations $\dot{x}_{t} = f_{\theta}^{\mathcal{T}}(x_{t})$ and $\dot{y}_{t} = f_{\theta}^{\mathcal{T}\to\mathcal{L}}(y_{t})$, such that $y_{t} = \psi_{\theta}(x_{t})$. Assume both $f_{\theta}^{\mathcal{T}}$ and $f_{\theta}^{\mathcal{T}\to\mathcal{L}}$ are continuously differentiable. Then, in the region $\mathcal{T}$, $x_{\mathrm{g}}$ is a globally asymptotically stable equilibrium of $f_{\theta}^{\mathcal{T}}$ if, $\forall t \in \mathbb{R}_{\geq0}$:
\begin{enumerate}
    \item $d_{t} = d_{t + \Delta t}$, \hspace{0.6cm} for $y_{0} =y_{\mathrm{g}}$,
    \item $d_{t} > d_{t + \Delta t}$, \hspace{0.5cm} $\forall y_{0}$ with $x_{0} \in \mathcal{T} \setminus \{x_{\mathrm{g}}\}$,
\end{enumerate}
where $\Delta t \in \mathbb{R}_{>0}$.
\end{theorem}
\begin{proof}
To prove this theorem, we demonstrate that if the surrogate stability conditions are satisfied, then the stability conditions from Theorem~\ref{theo:puma_stability} must also hold. This leads to $x_{\mathrm{g}}$ being a globally asymptotically stable equilibrium of $f^{\mathcal{T}}_{\theta}$.

\paragraph{Second stability condition of Theorem~\ref{theo:puma_stability}}
Let us begin by analyzing the fulfillment of the second stability condition from Theorem~\ref{theo:puma_stability}. This condition states that only $x_{\mathrm{g}}$ can map to $y_{\mathrm{g}}$ via $\psi_{\theta}$. Then, since $y_{\mathrm{g}}$ is defined as $\psi_{\theta}(x_{\mathrm{g}})$, we need to establish that no other $x_{t}$ maps to $y_{\mathrm{g}}$. To achieve this, we note that both $x_{t}$ and $x_{0}$ belong to the same state space $\mathcal{T}$. Therefore, showing that this statement holds $\forall x_{0} \in \mathcal{T}$ implies that it also holds $\forall x_{t} \in \mathcal{T}$.

Now, suppose for a contradiction that there exists some $x_0$ such that $x_0 \neq x_{\mathrm{g}}$ and $y_0 = \psi_{\theta}(x_0) = y_{\mathrm{g}}$. Then, according to the second surrogate condition, $d_{t} > d_{t + \Delta t}$, which contradicts the first surrogate condition. Consequently, the second stability condition must be satisfied.

\paragraph{First stability condition Theorem~\ref{theo:puma_stability}}
Let us now study the first stability condition. Our goal is to prove that $y_{\mathrm{g}}$ is a globally asymptotically stable equilibrium of $f_{\theta}^{\mathcal{T}\to\mathcal{L}}$ within the region $\mathcal{L}$, where the system is defined. Surrogate condition 2 hints that the system's stability could be verified through a Lyapunov candidate defined using $d_{t}$. This is because the condition enforces the distance $d_{t}$ to strictly decrease within the interval defined by $\Delta t$. However, this does not necessarily imply that the Lyapunov candidate strictly decreases with time, as it is possible for it to strictly increase locally while adhering to this condition, as exemplified in Fig. \ref{fig:surrogate_upper}. As a result, we proceed to demonstrate the global asymptotic stability of $y_{\mathrm{g}}$ using Theorem~\ref{theo:kl_stability}. Specifically, we aim to do so by employing a class-$\mathcal{KL}$ upper bound $\beta$ that fulfills~\eqref{eq:beta_condition}.
\begin{figure}[t]
    \centering
    \includegraphics[width=\columnwidth]{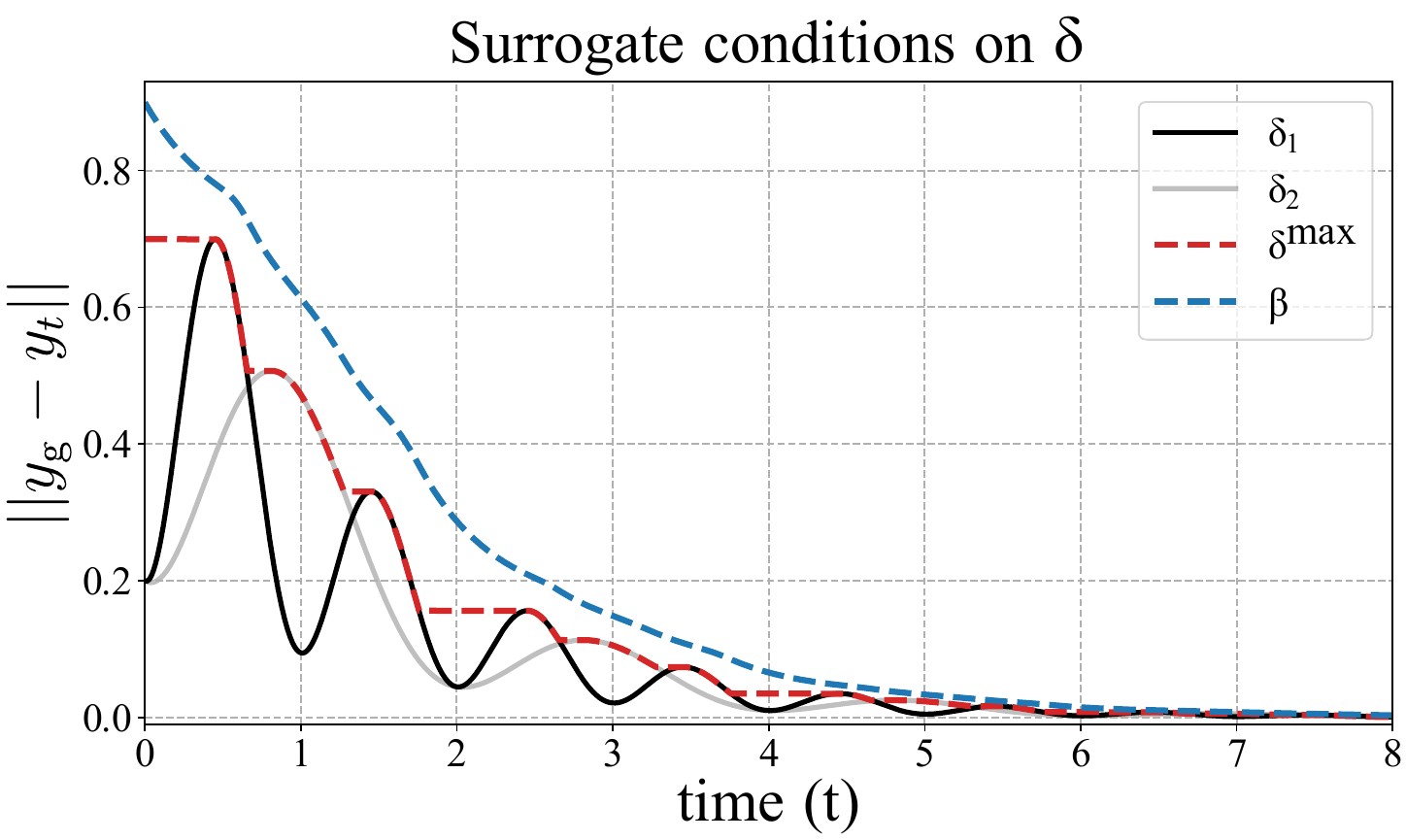}
    \caption{Time evolution of two functions $\delta$, $\delta_{1}$ and $\delta_{2}$, starting with different initial conditions $y_{0}$, but same $d_{0}$. Both satisfy the surrogate stability conditions with ${\Delta t = 2}$. Additionally, the values of $\delta^{\text{max}}$ and $\beta$, computed using these functions, are shown. In this representation, ${\delta = e^{-a\cdot t}\left(\sin^{2}\left(\omega t\right) + d_{0} \cdot \cos^{2}\left(\omega t\right)\right)}$ with $a = 0.75$, $d_{0} = 0.2$, and $\omega = [\pi, \pi/2]$.}
    \label{fig:surrogate_upper}
\end{figure}

To achieve this, we first define an evolution function for the distance $d_{t}$, for a given $y_{0}$ and $t$, as $\delta:\mathcal{L} \times \mathbb{R}_{\geq 0} \to \mathbb{R}_{\geq 0}$, with $\delta(y_{0}, t) = ||y_{\mathrm{g}} - \Phi^{y}_{\theta}(t,y_{0})||$. Here, $\Phi^{y}_{\theta}$ represents the evolution function of $y_{t}$ under the dynamical system $f^{\mathcal{T} \to \mathcal{L}}_{\theta}$. Then, we can express the upper bound $\beta$ as
\begin{equation}
    \delta(y_{0}, t) \leq \beta(d_{0}, t), 
\end{equation}
where $d_{0} = \delta(y_{0}, 0)$. Recall that for ensuring asymptotic stability, $\beta$ must also satisfy the following properties:
i) $\beta(0, t) = 0$,
ii) $\beta$ weakly decreases with $t$,
iii) $\beta$ is continuous with respect to $d_{0}$ and $t$,
iv) $\beta \to 0$ as $t \to \infty$, and
v) $\beta$ strictly increases with $d_{0}$.

It is important to note that verifying all these properties can lead to a lengthy proof. Thus, while we introduce $\beta$ and discuss the key concepts behind its design here, the complete proof and details are provided in Proposition~\ref{prop:existence_KL}, Appendix \ref{app:upperbound}. 

We now proceed to describe the three main aspects considered in the design of $\beta$.
\begin{enumerate}[leftmargin=*]
    \item \emph{Upper bound in time:} First, we need to identify a $\beta$ that weakly decreases with time. For this purpose, we introduce a function that computes the maximum of $\delta$ over the window $[t, t + \Delta t]$ for any given $y_{0}$ and $t$, i.e.,
    \begin{equation}
        \delta^{\text{max}}_{t + \Delta t}(y_{0}, t) = \max_{s \in [t, t + \Delta t]} \delta(y_{0}, s).
    \end{equation}
    Clearly, this function serves as an upper bound for $\delta$. Moreover, this function must weakly decrease with time. Considering surrogate condition 2, which indicates that $\forall s \in [t, t + \Delta t]$, we have $\delta(y_{0}, s + \Delta t) < \delta(y_{0}, s)$, it follows that there is no $\delta$ greater than $\delta^{\text{max}}_{t + \Delta t}$ in the interval $[t + \Delta t, t + 2\Delta t]$. By extending this observation for every interval $[t + n \cdot \Delta t, t + (n + 1) \cdot \Delta t]$, with $n \in \mathbb{N}$, we can conclude that $\delta^{\text{max}}_{t + \Delta t}$ weakly decreases with time.

    \item \emph{Upper bound in space:} The function $\delta^{\text{max}}_{t + \Delta t}(y_{0}, t)$ provides an upper bound of $\delta$ for a given $y_{0}$. However, $\beta(d_{0}, t)$ depends on $d_{0}$ rather than directly on $y_{0}$. To address this, for any specified $d_{0}$, we must ensure that $\beta \geq \delta$ for every $y_{0}$ located at this particular distance from the equilibrium. The set of initial conditions $y_{0}$ fulfilling this condition can be defined as $\mathcal{Y}_{0}(d_{0}) = \{y_{0} \in \mathcal{L}: ||y_{\mathrm{g}} - y_{0}|| = d_{0}\}$. Considering this, we can introduce an upper bound dependent on $d_{0}$ by computing the maximum of each $\delta^{\text{max}}_{t + \Delta t}(y_{0}, t)$, where $y_{0} \in \mathcal{Y}_{0}$:
    \begin{equation}
         \delta^{\text{max}}(d_{0}, t) = \max_{y_{0} \in \mathcal{Y}_{0}(d_{0})}\left(\delta_{t + \Delta t}^{\text{max}}(y_{0}, t)\right).
    \end{equation}

    \item \emph{Strictly increasing/decreasing function:} Similarly to $\delta^{\text{max}}_{t + \Delta t}$, the function $\delta^{\text{max}}$ also weakly decreases as a function of time (see Appendix \ref{app:upperbound} for details). This is not inherently problematic, as it verifies the properties of class-$\mathcal{KL}$ functions. However, $\beta$ must strictly increase as a function of $d_{0}$, and, for this to be the case, in our formulation, $\beta$ must also strictly decrease as a function of time, $\forall d_{0} \in \mathcal{L} \setminus \{y_{\mathrm{g}}\}$.

    To achieve this, we consider $\beta$ as the evolution function of a first-order linear dynamical system with state $z_{t}$, using $\delta^{\text{max}}$ as the reference. This leads to the equation
    \begin{equation}
        \dot{z}_{t} = \alpha(z_{t} - \delta^{\text{max}}),
    \end{equation}
    where $\alpha < 0$. With an initial condition $z_{0}$ greater than $ \delta^{\text{max}}_{0} = \delta^{\text{max}}(d_{0}, 0)$, we ensure that $\beta$ strictly decreases with time. Moreover, $\beta$ remains above $\delta^{\text{max}}$, and consequently above $\delta$, for all $d_{0} \in \mathcal{L} \setminus \{y_{\mathrm{g}}\}$. Hence, by choosing ${z_{0} = \delta^{\text{max}}_{0} + d_{0}} $, we ensure that $\beta$ exhibits the desired increasing/decreasing properties. We can express $\beta$ as
    \begin{equation}
        \beta(d_{0}, t) = z_{0} + \int_{0}^{t}\dot{z}_{s} ds.
    \end{equation}
\end{enumerate}
Based on Proposition~\ref{prop:existence_KL}, Appendix \ref{app:upperbound}, we can confirm that our formulation for $\beta$ is a valid class-$\mathcal{KL}$ upper bound of $\delta$. This indicates that the first stability condition of Theorem~\ref{theo:puma_stability} is satisfied, concluding our proof. 
\end{proof}

\subsubsection{Loss Function}
Crucially, the surrogate stability conditions from Theorem~\ref{theo:puma_surrogate} can be enforced in a DNN by minimizing the following expression, $\forall y_{0} \in \mathcal{L}$ and $\forall t \in \mathbb{R}_{\geq0}$,
\begin{equation}
\label{eq:preliminary_stable_loss}
    \max\bigl(0, m + d(y_{\mathrm{g}}, y_{t + \Delta t}) - d(y_{\mathrm{g}}, y_{t})\bigr),
\end{equation}
where $m \in \mathbb{R}_{>0}$. This function resembles the form of the triplet loss introduced in Sec. \ref{sec:triplet_loss}. However, in our context, $y_{\mathrm{g}}$ serves as the anchor sample, $y_{t + \Delta t}$ as the positive sample, and $y_{t}$ as the negative sample. In the following discussion, we explore how this expression induces the fulfillment of the surrogate stability conditions within a DNN.

\paragraph{Second surrogate condition}
\begin{figure}[t]
    \centering
    \includegraphics[width=\columnwidth]{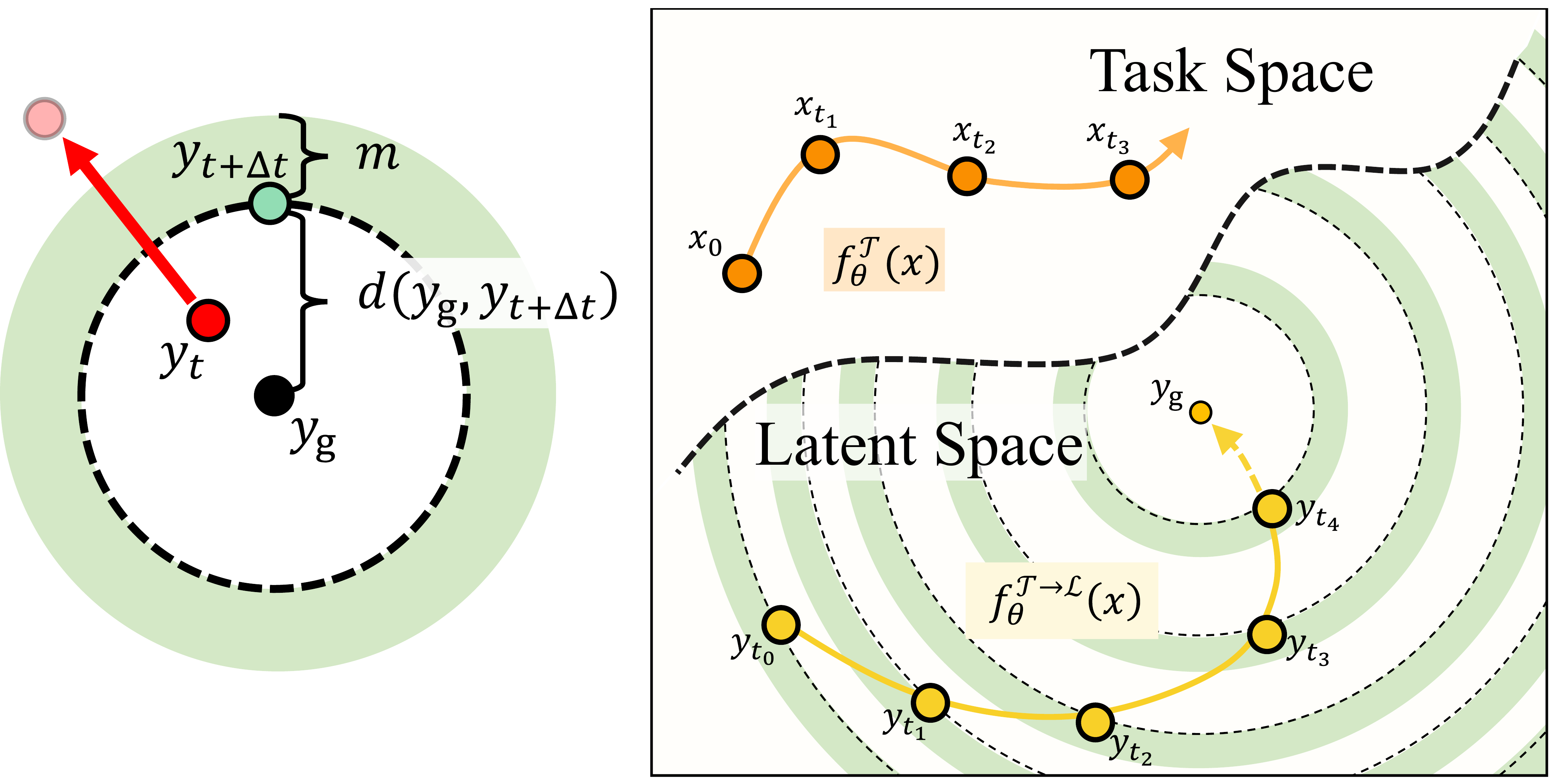}
    \caption{Left: Illustration of the effect of optimizing $\ell_{\text{stable}}$. The red arrow depicts $y_{t}$ and $y_{t+\Delta t}$ being modified to fulfill \eqref{eq:triplet_inequality}. Right: Example of trajectories generated with $f^{\mathcal{T}}_{\theta}$ and $f^{\mathcal{T}\to\mathcal{L}}_{\theta}$ \textbf{post-training}. Each $y_{t+\Delta t}$ is closer to $y_{\mathrm{g}}$ than its predecessor $y_{t}$.}
    \label{fig:task_latent}
\end{figure}

For any $y_{0}$ where $x_{0} \neq x_{\mathrm{g}}$, minimizing \eqref{eq:preliminary_stable_loss} enforces transitions to progressively approach $y_{\mathrm{g}}$, as its minimization implies
\begin{equation}
    d(y_{\mathrm{g}}, y_{t + \Delta t}) + m \leq d(y_{\mathrm{g}}, y_{t}).
    \label{eq:triplet_inequality}
\end{equation}
Hence, $d(y_{\mathrm{g}}, y_{t + \Delta t})<d(y_{\mathrm{g}}, y_{t})$, i.e., the second surrogate condition (see Fig. \ref{fig:task_latent}). 

\paragraph{First surrogate condition}
In the case where $y_0 = y_{\mathrm{g}} = \psi_{\theta}(x_{\mathrm{g}})$, that is, when the system is initialized at the equilibrium, enforcing a value of $y_{t + \Delta t}$ different from $y_{\mathrm{g}}$ would only increase the function in~\eqref{eq:preliminary_stable_loss}. Consequently, for $y_0 = y_{\mathrm{g}} = \psi_{\theta}(x_{\mathrm{g}})$, \eqref{eq:preliminary_stable_loss} achieves its minimum when $y_{t + \Delta t} = y_{\mathrm{g}} = y_0$, equal to $m$. Therefore, this equation also enforces the first surrogate condition.

Finally, since $\delta(y_{0}, t) = d(y_{\mathrm{g}}, y_{t})$, we present the stability loss function as
\begin{equation}
\label{eq:stable_loss}
    \ell_{\text{stable}} = \sum_{y_{0}\in \mathcal{B}^{s}}\sum_{t \in \mathcal{H}^{s}} \max(0, m + \delta(y_{0}, t + \Delta t) - \delta(y_{0}, t)).
\end{equation}
In this equation, $\mathcal{B}^{s}$ denotes a batch of initial latent states $y_{0}$. These states are derived by mapping initial states $x_{0}$ (\textbf{sampled randomly} from $\mathcal{T}$) via the function $\psi_{\theta}$. Meanwhile, $\mathcal{H}^{s}$ represents a set of time instants $t$ at which the loss is minimized.

To compute this loss, we must recall that the evolution function of $y_{t}$ is represented as $\Phi^{y}_{\theta}$. Consequently, 
\begin{equation}
    \delta(y_{0}, t) = ||y_{\mathrm{g}} - \Phi^{y}_{\theta}(t, y_{0})||.
\end{equation}
Given the absence of an analytical representation for $\Phi^{y}_{\theta}$, we approximate it using the forward Euler method, and optimize it using BPTT, in a similar manner to Sec. \ref{sec:behavioral_cloning}. 

Importantly, optimizing this loss for all $t \in \mathbb{R}_{\geq 0}$ using BPTT is not feasible, as it would necessitate computing the loss over an infinite number of samples. Nevertheless, this limitation is not a significant concern when dealing with time-invariant dynamical systems. In such systems, any state $y_{t}$ can be equivalently represented by an initial condition $y_{0}$, since both reside in the same state space $\mathcal{L}$. Consequently, when states are randomly sampled from $\mathcal{T}$, and thereby from $\mathcal{L}$, we are essentially sampling from the space of all possible $y_{t}$.

\subsection{On the Stability Loss Metric}
Intentionally, we have not specified which distance function or metric should be used for computing the stability loss, since it should be selected depending on the geometry employed to describe the robot's state space. 
\begin{figure}[t]
    \centering
    \includegraphics[width=0.9\columnwidth]{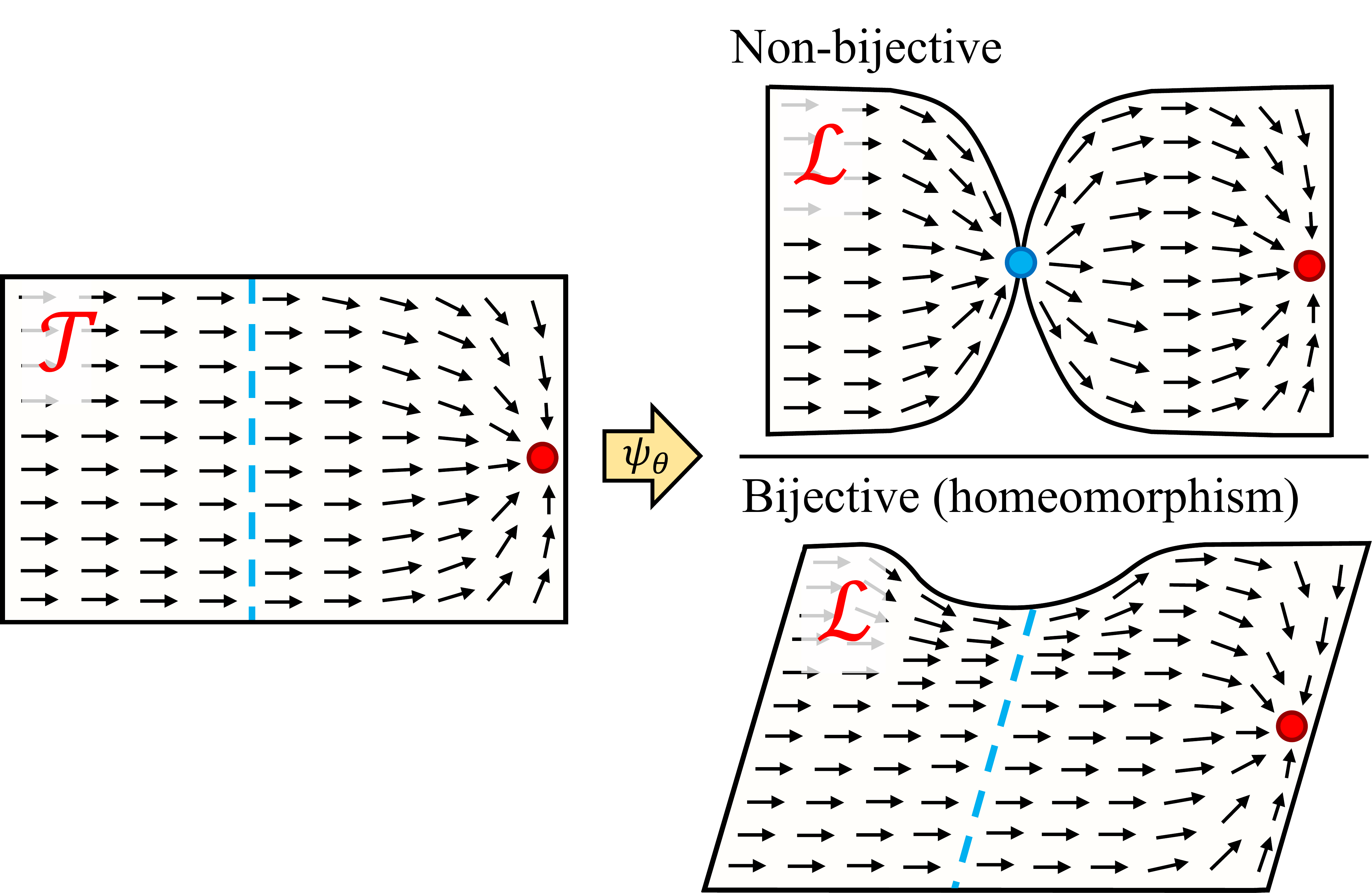}
    \caption{Non-bijective vs bijective solution. In the non-bijective case, every point on the blue line in $\mathcal{T}$ maps to the blue point in $\mathcal{L}$. In the bijective case, the original line undergoes a transformation that both compresses and skews it, yet a one-to-one mapping is maintained.
}
    \label{fig:non_bijective}
\end{figure}

To elaborate, recall that the learned dynamical system can be expressed as $\dot{x}_{t} = \phi_{\theta}(y_{t})$. It then follows that the output of our model is entirely determined by the latent state $y_{t}$. This implies that if two different states $x^{\text{a}}_{t}$ and $x^{\text{b}}_{t}$ map to the same latent state, that is, $y_{t} = \psi_{\theta}(x^{\text{a}}_{t}) = \psi_{\theta}(x^{\text{b}}_{t})$, the time derivative computed by the learned dynamical system for both states will be identical. Such behavior would manifest in certain states if $\psi_{\theta}$ were a non-bijective\footnote{More rigorously, the property being described is that of non-injective functions. However, for practical purposes, we can define the codomain of $\psi_{\theta}$ to be equal to its image, which is the property that an injective function must have to be bijective. Hence, in this specific context, we use these terms interchangeably.} function. It would, therefore, be potentially harmful for the learning process to enforce $\psi_{\theta}$ to be non-bijective, as this would constrain the family of solutions to which the system can converge, hindering the DNN's optimization process. Ideally, we would like $\psi_{\theta}$ to have the capacity to converge to a bijective function if required, where each state $x_{t}$ maps to a unique latent state $y_{t}$. Consequently, for any state $x_{t}$, it would be possible to compute a value of $\dot{x}_{t}$ that is independent of those calculated for any other state. Fig.~\ref{fig:non_bijective} presents an example of bijective and non-bijective mappings between dynamical systems where the surrogate stability conditions are enforced.

It turns out that the capacity of $\psi_{\theta}$ to converge to a bijective function is closely linked to the metric used in calculating $\ell_{\text{stable}}$, which should be chosen based on the geometry of the robot's state space. In the following subsection, we explore this relationship in more depth.

\subsubsection{Homeomorphisms and State Space Geometry}
To study under which conditions $\psi_{\theta}$ can converge to a bijective function, let us introduce the concept of \emph{topologically equivalent} manifolds. Two manifolds, e.g., $\mathcal{T}$ and $\mathcal{L}$, are topologically equivalent if a continuous and bijective mapping, known as an \emph{homeomorphism}, exists between them \citep{needham2021visual}. In other words, two topologically equivalent manifolds can be \emph{stretched}, \emph{compressed}, or \emph{twisted} into each other without \emph{tearing} or \emph{gluing} space. Since DNNs are continuous functions\footnote{We can assume this since every broadly used model is continuous \citep{goodfellow2016deep,geron2022hands}.}, a bijective function $\psi_{\theta}$ would also serve as a homeomorphism\footnote{DNNs are commonly continuously differentiable, so, sometimes in the literature the term \emph{diffeomorphism} is employed instead, as diffeomorphisms are continuously differentiable homeomorphisms.} between $\mathcal{T}$ and $\mathcal{L}$.

Moreover, since we have a notion of distance in both $\mathcal{T}$ and $\mathcal{L}$, it follows that these manifolds are metric spaces. A key property of metric spaces is that their topology is generated by their distance functions. Therefore, if two metric spaces are topologically equivalent, their metrics are termed as being \emph{equivalent} \citep{deza2009encyclopedia}. It is important to clarify that this does not necessarily mean that the two distance functions are identical, but rather that they induce metric spaces that are homeomorphic to each other.

In our context, this implies that the distance function employed in the stability loss \eqref{eq:stable_loss} must induce a topology in $\mathcal{L}$ that is equivalent to that of $\mathcal{T}$.  For example, if orientations are described using unit quaternions, the topology of $\mathcal{T}$ would be spherical. Then, the stability loss metric should generate a topology that is homeomorphic to the sphere. Otherwise, it would be infeasible for the DNN to establish a homeomorphism between $\mathcal{T}$ and $\mathcal{L}$. 

In this work, Sec. \ref{sec:experiments}, we use unit quaternions to represent orientations in our robot experiments. For a detailed discussion on metrics and pose control in this context, the reader is referred to Appendix \ref{learning_spheres}.

\subsection{Boundary Conditions}\label{sec:boundary_conditions}
Lastly, it is crucial to ensure that a dynamical system evolving in $\mathcal{T}$ always remains within this manifold. Two scenarios are relevant to this work: 1) ensuring $\mathcal{T}$ is positively invariant with respect to $f^{\mathcal{T}}_{\theta}$; and 2) considering the state's geometry when computing the evolution of the dynamical system.

\subsubsection{Positively Invariant Sets}
In PUMA, the stability of a motion is enforced by randomly sampling points from $\mathcal{T}$ and minimizing $\ell_{\text{stable}}$. Thus, stability cannot be ensured in regions where this loss is not minimized, i.e., outside of $\mathcal{T}$. When boundaries are imposed on the robot's workspace, the learned dynamical system $f^{\mathcal{T}}_{\theta}$ can potentially evolve towards these boundaries, leaving $\mathcal{T}$. Hence, to ensure stability, a state evolving within $\mathcal{T}$ must not leave $\mathcal{T}$. In other words, $\mathcal{T}$ has to be a \emph{positively invariant set} with respect to $f^{\mathcal{T}}_{\theta}$ \citep{lemme2014neural,khalil2002nonlinear}.

To address this, we design the dynamical system so that it cannot leave $\mathcal{T}$ by construction. This can be achieved by projecting any transitions that would leave $\mathcal{T}$ back onto its boundary. For example, in Euclidean state spaces, $\mathcal{T}$ can be represented as a hypercube; therefore, in this case, this projection is achieved by saturating/clipping the points that leave $\mathcal{T}$. Furthermore, we found that introducing an additional loss, denoted as $\ell_{\partial}$, can be beneficial in enforcing this condition through the optimization process of the DNN. As proposed in \citet{lemme2014neural}, this can be accomplished by using the scalar product between the dynamical system's velocity $v(x_{t})$ (which equates to $f^{\mathcal{T}}_{\theta}$ for first-order systems) and the outward-pointing normal vector $n(x_{t})$ at states within the boundary of $\mathcal{T}$. This product should be ensured to be equal to or less than zero. Then, to achieve this for DNNs, we introduce the following loss function
\begin{equation}
\ell_{\partial} = \max(0, n(x_{t}) \cdot v(x_{t})).
\end{equation}

\subsubsection{Evolving in Non-Euclidean State Spaces}
\label{sec:boundary_non_euc}
When modeling dynamical systems in non-Euclidean state spaces, it is crucial to ensure that states do not evolve outside the manifold representing them. For instance, unit quaternions must remain within the unit sphere. However, if these states are evolved using Euclidean geometry tools, such as the forward Euler integration method, deviations from the manifold are likely to occur. 

Non-Euclidean state spaces are commonly defined within a higher-dimensional Euclidean space and can sometimes be constructed by incorporating constraints into this space. Consider the unit quaternion as an example: its state space consists of every vector in $\mathbb{R}^{4}$ with a unit norm, forming the 3-sphere $\mathcal{S}^{3} \subseteq \mathbb{R}^{4}$. Therefore, in such scenarios, by integrating an operation that enforces these constraints, such as normalization for $\mathcal{S}^{3}$, into the structure of the DNN representing $f^{\mathcal{T}}_{\theta}$, we ensure that transitions remain within the manifold. Moreover, in this way, the distortions that this operation introduces into the dynamical system's output are factored into the DNN's optimization process, ensuring they are accounted for.

Alternatively, when we have access to the \emph{Riemannian metric} of a state space manifold, it is possible to use the \emph{exponential} and \emph{logarithmic maps} to do Euclidean calculus in the \emph{tangent bundle} of the manifold, and then map the solution back on the manifold (the reader is referred to \cite{sommer2020introduction} for more details).



%% file: sections/05_experiments.tex
\section{Experiments}
\label{sec:experiments}
We validate our method with three datasets, each allowing us to study different aspects of it. For evaluation purposes, we use these datasets under the assumption of perfect tracking of the desired state derivatives, $\dot{x}_{t}^{d}$, provided by $f^{\mathcal{T}}_{\theta}$, without any involvement of robots in this process. Subsequently, we test our method in two real-world settings using two different robots. We have made our code implementation of PUMA publicly available at: \url{https://github.com/rperezdattari/Deep-Metric-IL-for-Stable-Motion-Primitives}. Details on the DNN's hyperparameters optimization process are presented in Appendix \ref{appendix:hyperparam}.

\subsection{Euclidean Datasets}
First, we evaluate our method using datasets of Euclidean motions. This approach allows us to examine the performance of different variations of PUMA for Euclidean dynamical systems and compare their effectiveness against state-of-the-art methods.
\subsubsection{LASA}\label{sec:exps_lasa}
The LASA dataset \citep{khansari2011learning} is composed of 30 human handwriting motions, each consisting of 7 demonstrations of desired trajectories under different initial conditions. These demonstrations are two-dimensional and designed to be modeled using first-order systems, i.e., output desired velocities as a function of their positions. To compare accuracy performance between different models, we employ the same metrics in every experiment: 1) Root Mean Squared Error (RMSE), Dynamic Time Warping Distance (DTWD) \citep{muller2007dynamic} and Fr{\'e}chet Distance (FD) \citep{eiter1994computing}.
\begin{figure}[t]
\centering
    \subfloat[][Comparison of different variations of PUMA.]{\includegraphics[width=0.49\linewidth]{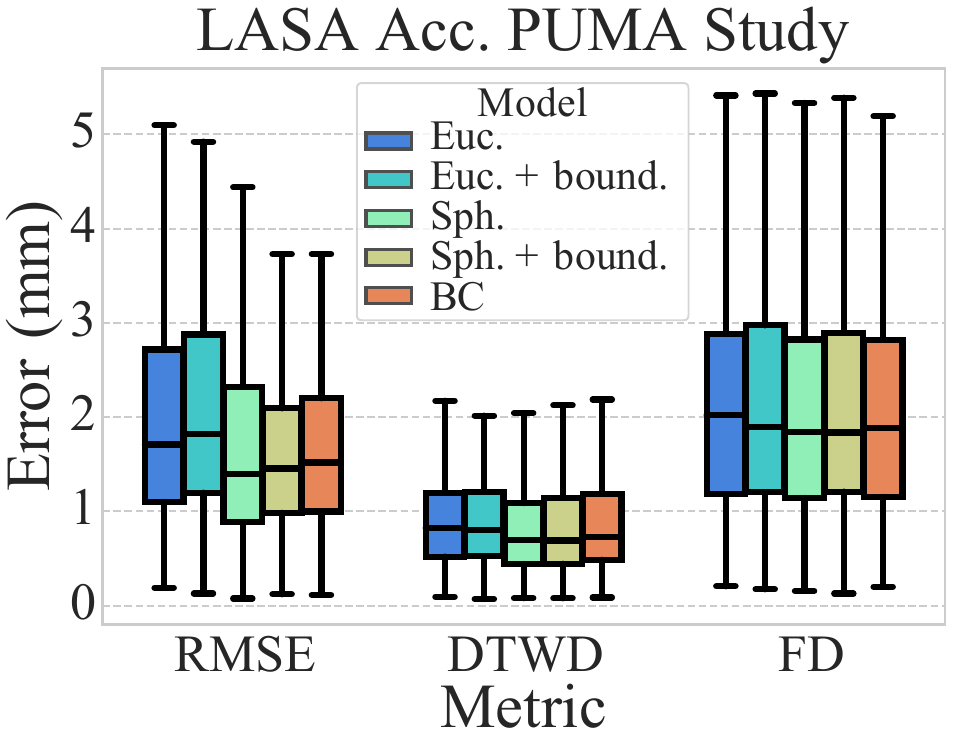}\label{fig:LASA_plots_a}}\hspace{0.03cm}
    \subfloat[][Comparison of PUMA with other state-of-the-art (SoA) methods.]{\includegraphics[width=0.49\linewidth]{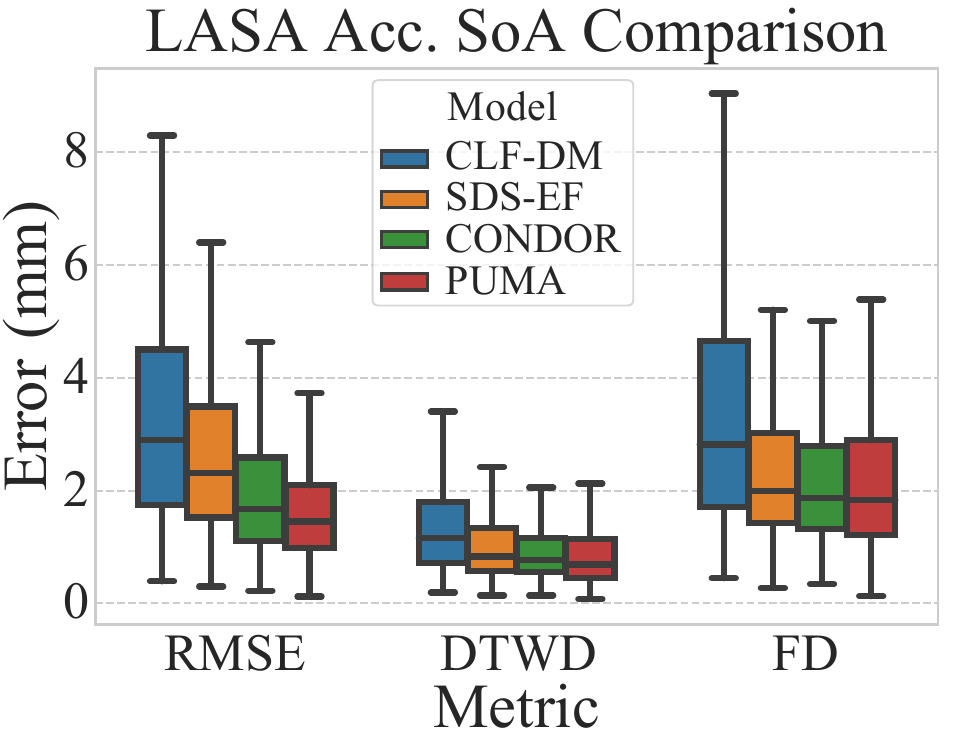}\label{fig:LASA_plots_b}}
\caption{Accuracy study in the LASA dataset.}
\label{fig:LASA_plots}
\end{figure}

\paragraph{Accuracy}
Fig. \ref{fig:LASA_plots_a} showcases the accuracy of four variations of PUMA. Here, we compare the performance of different distance metrics in $\ell_{\text{stable}}$ for motions in Euclidean spaces, focusing on the Euclidean distance and the great-circle (spherical) distance. Furthermore, we also examine the influence of $\ell_{\partial}$ on the accuracy of the learned motions. We use Behavioral Cloning (BC) without a stability loss as an upper performance bound for comparison. Interestingly, each variation of PUMA achieves a similar performance; however, PUMA with a spherical metric shows slightly better performance than PUMA with an Euclidean metric. This suggests that the DNN has no difficulties in mapping the Euclidean space $\mathcal{T}$ into a spherical space $\mathcal{L}$. Lastly, we can also observe that the use of $\ell_{\partial}$ does not harm the accuracy performance of PUMA.

\paragraph{Stability}
The stability of the motions learned by PUMA hinges on the successful minimization of \eqref{eq:puma_loss}, which we need to test after the learning process concludes empirically. To do this, we integrate the dynamical system over $L$ time steps, starting from $P$ initial states, and observe whether the system converges to the goal. The larger the $P$, the more accurate our results. If $L$ is sufficiently large, the system should reach the goal after $L$ steps. By measuring the distance between the last state visited and the goal, and confirming that it falls below a pre-set threshold $\epsilon$, we can evaluate if a trajectory is successful (i.e., it converges to the goal).

Table \ref{tab:table_lasa} provides the stability results of BC and different PUMA variations. The data shows that every variation of PUMA successfully enforces stability across the dataset, generating no unsuccessful trajectories. Compared to BC, which has a 36.4653\% rate of unsuccessful trajectories, the benefit of the proposed loss in enforcing stability becomes clear.

\begin{table}[t]
\centering
\caption{Percentage of unsuccessful trajectories over the LASA dataset ($L$ = 2500, $P$ = 2500, $\epsilon$ = 1mm).}
\label{tab:table_lasa}
\resizebox{\columnwidth}{!}{%
\begin{tabular}{ccccc}
\hline
\begin{tabular}[c]{@{}c@{}}Behavioral\\ Cloning\end{tabular} & \begin{tabular}[c]{@{}c@{}}PUMA\\ (Euc.)\end{tabular} & \begin{tabular}[c]{@{}c@{}}PUMA\\ (Euc. + $\ell_{\partial}$)\end{tabular} & \begin{tabular}[c]{@{}c@{}}PUMA\\ (Sph.)\end{tabular} & \begin{tabular}[c]{@{}c@{}}PUMA\\ (Sph. + $\ell_{\partial}$)\end{tabular} \\ \hline
36.4653\%                                                   & \textbf{0.0000\%}                                    & \textbf{0.0000\%}                                                        & \textbf{0.0000\%}                                    & \textbf{0.0000\%}                                                        \\ \hline
\end{tabular}
}
\end{table}

\paragraph{State-of-the-art comparison}
Fig. \ref{fig:LASA_plots_b} presents an accuracy comparison of PUMA with other state-of-the-art methods, namely: 1) Control Lyapunov Function-based Dynamic Movements (CLF-DM) using Gaussian Mixture Regression (GMR) \citep{khansari2014learning}, 2) Stable Dynamical System learning using Euclideanizing Flows (SDS-EF) \citep{rana2020euclideanizing}, and 3) CONDOR. The results for PUMA correspond to the best-performing variation in this dataset, i.e., spherical distance with $\ell_{\partial}$. We observe that PUMA achieves competitive results, demonstrating similar performance in DTWD and FD to CONDOR and SDS-EF, and slightly superior performance under RMSE.
\begin{figure}[t]
\centering
    \subfloat[][\emph{S} shape PUMA.]{\includegraphics[width=0.49\linewidth]{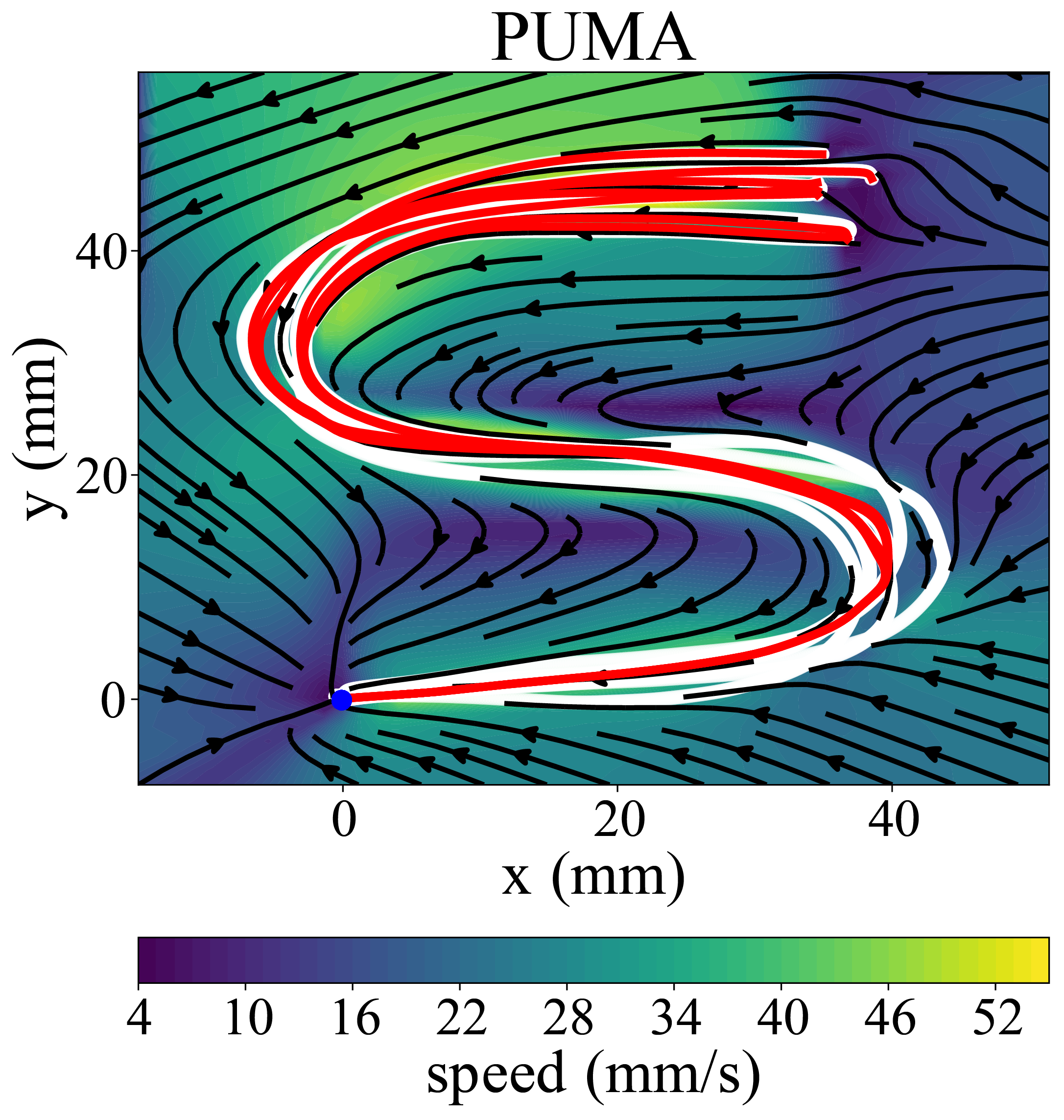}\label{fig:boundary_loss_a}}
    \subfloat[][\emph{S} shape CONDOR.]{\includegraphics[width=0.49\linewidth]{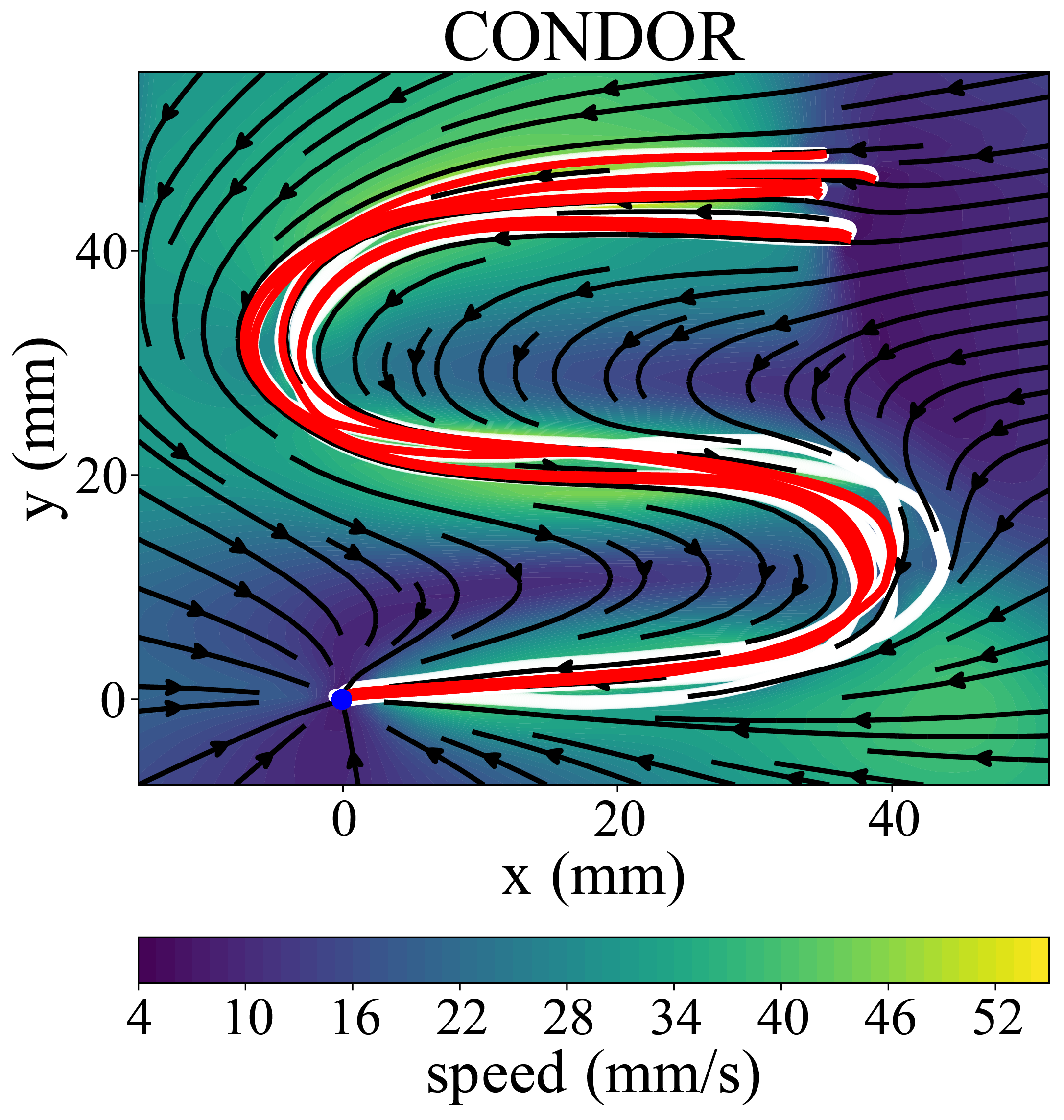}\label{fig:boundary_loss_b}}\newline
    \subfloat[][\emph{S} shape PUMA with $\ell_{\partial}$.]{\includegraphics[width=0.49\linewidth]{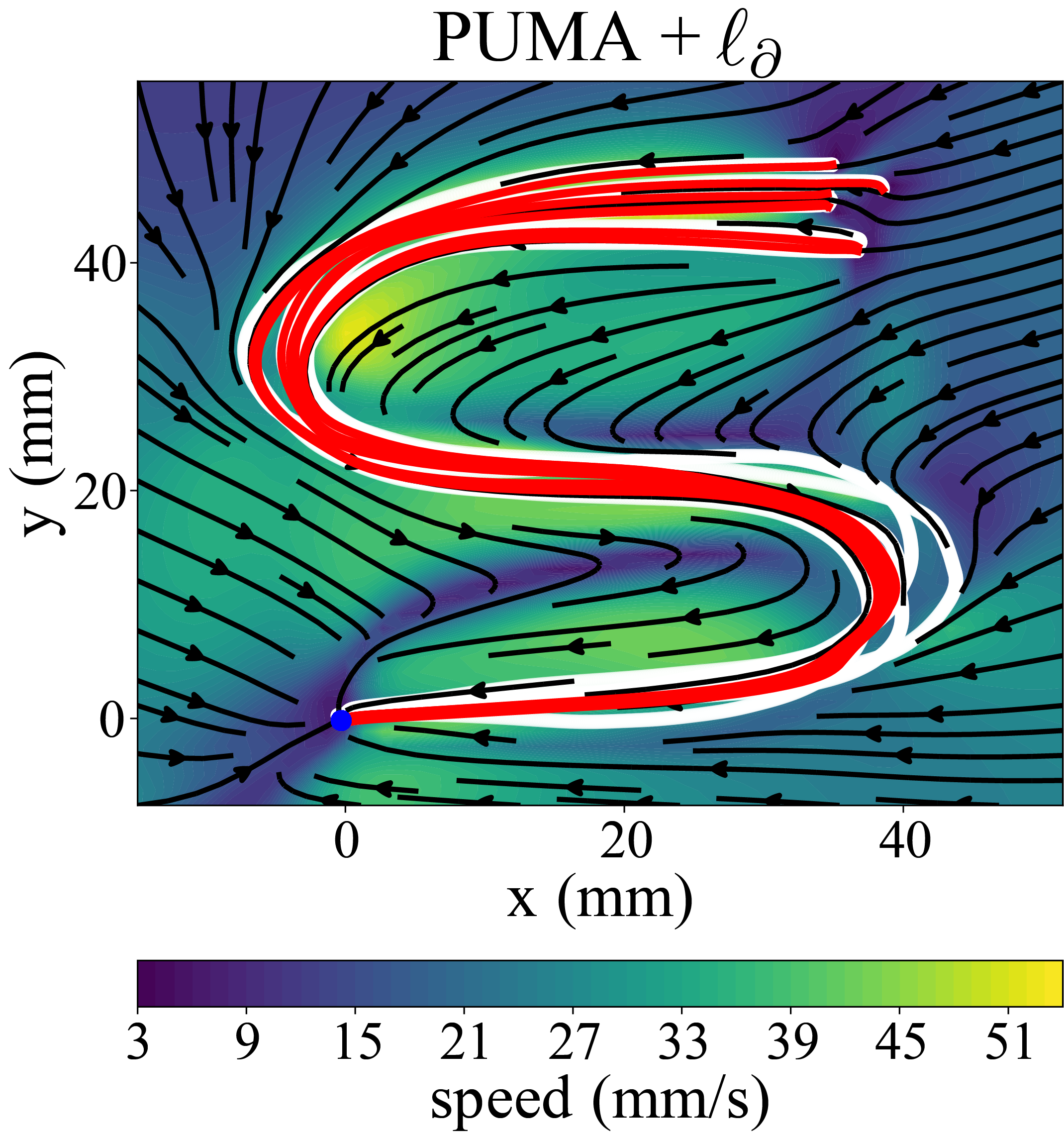}\label{fig:boundary_loss_c}}
    \subfloat[][Positive invariance comparison.]{\includegraphics[width=0.5\linewidth]{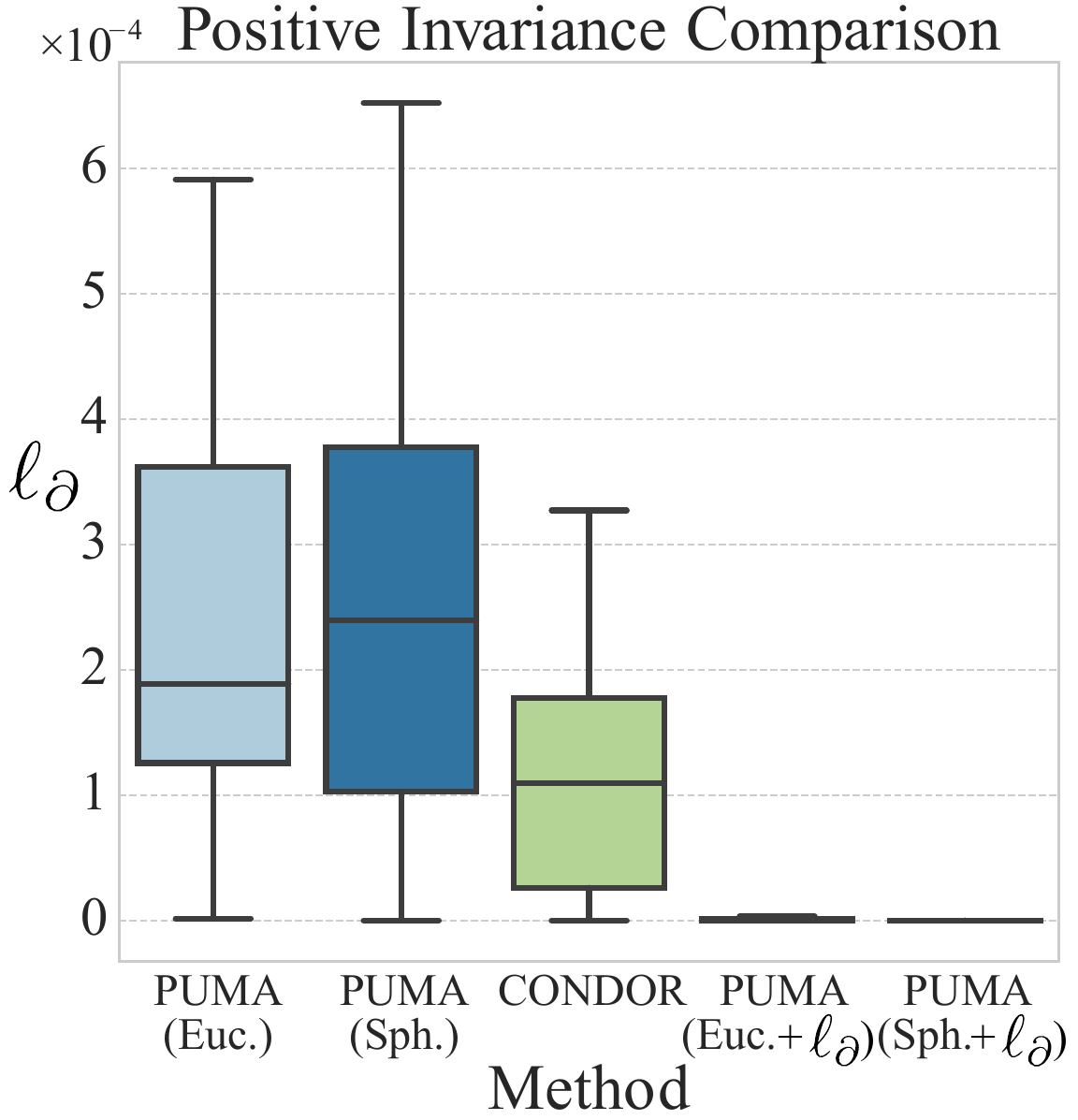}\label{fig:boundary_loss_d}}
\caption{Positive invariance evaluation. In figures (a)-(c), white curves represent demonstrations. Red curves represent learned motions when starting from the same initial conditions as the demonstrations. The arrows indicate the vector field of the learned dynamical system. Figure (d) provides a positive invariance comparison between CONDOR and the variations of PUMA.}
\label{fig:boundary_loss}
\end{figure}

\paragraph{Boundary loss}
Fig. \ref{fig:boundary_loss} demonstrates the effect of the boundary loss $\ell_{\partial}$ in PUMA. Figures \ref{fig:boundary_loss_a} and \ref{fig:boundary_loss_b} present an example of the qualitative performance of PUMA and CONDOR when no boundary loss is applied. In the top-left region of these images, PUMA exhibits non-smooth trajectories at the boundary, which abruptly change direction due to the applied saturation (see Sec. \ref{sec:boundary_conditions}). Conversely, CONDOR learns a smoother trajectory in this region. This feature in PUMA vanishes when $\ell_{\partial}$ is applied, as depicted in Fig. \ref{fig:boundary_loss_c}. Finally, Fig.~\ref{fig:boundary_loss_d} provides a quantitative evaluation of the boundary loss, confirming our qualitative observations.

In this work, this loss is \textbf{only relevant for Euclidean state spaces} because we do not introduce boundaries in the state space when controlling orientation in $\mathcal{S}^{3}$.

\subsubsection{LAIR}
Contrary to the LASA dataset, the LAIR dataset \citep{perez2023stable} is specifically designed to evaluate second-order motions, i.e., those that map current position and velocity to acceleration. This dataset comprises 10 human handwriting motions, with the state being 4-dimensional, encompassing a 2-dimensional position and velocity. The dataset's shapes contain multiple position intersections, intentionally designed to necessitate the use of at least second-order systems for their successful modeling.

\paragraph{CONDOR / PUMA comparison}
The state-of-the-art methods outlined in Sec.~\ref{sec:exps_lasa}, excluding CONDOR, do not address the challenge of learning stable second-order systems. This can be attributed to the greater difficulty inherent in learning such systems compared to first-order systems. As a result, we compare the performance of PUMA with that of CONDOR. Fig.~\ref{fig:LAIR_plots_a} illustrates the comparative accuracy of both methods, with PUMA surpassing CONDOR on all metrics. PUMA's greater learning flexibility allows it to converge to solutions beyond CONDOR's capabilities. Moreover, this enhanced flexibility, as depicted in Fig.~\ref{fig:LAIR_plots_b}, endows PUMA with a more stable learning process. In practice, this makes a significant difference, as motions with unsuccessful trajectories must be discarded when considering the system's stability. Therefore, greater learning stability results in a larger set of motions without unsuccessful trajectories, providing more alternatives of successfully learned systems to choose from.
\begin{figure}[t]
\centering
    \subfloat[][Accuracy comparsion.]{\includegraphics[width=0.5\linewidth]{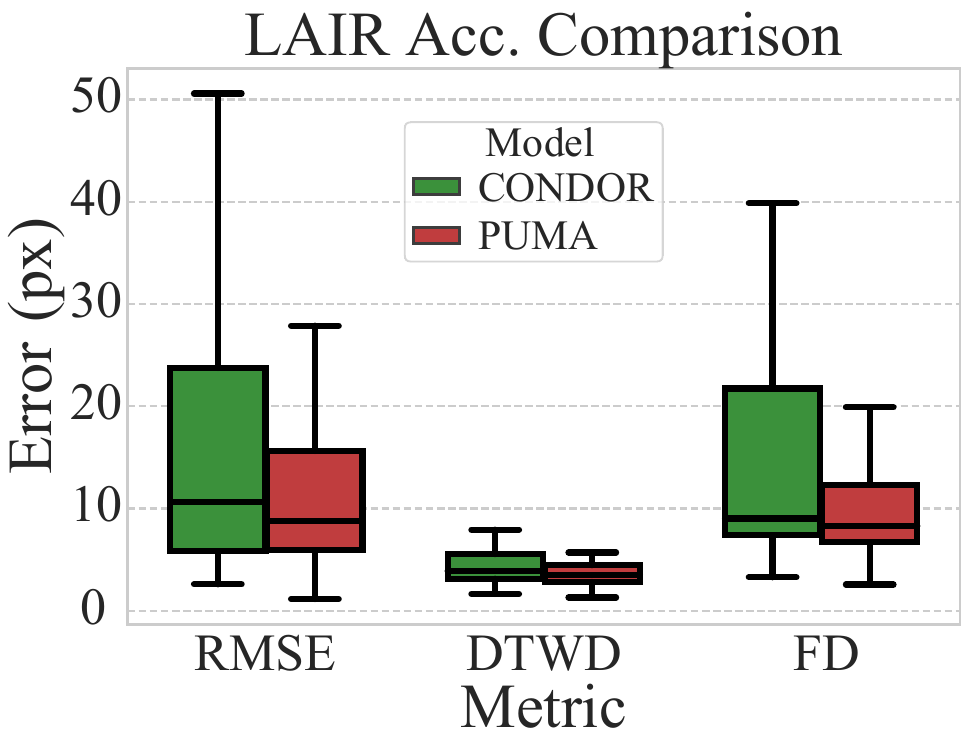}\label{fig:LAIR_plots_a}}\hspace{0.05cm}
    \subfloat[][Training stability comparison measured as the average distance to the goal.]{\includegraphics[width=0.48\linewidth]{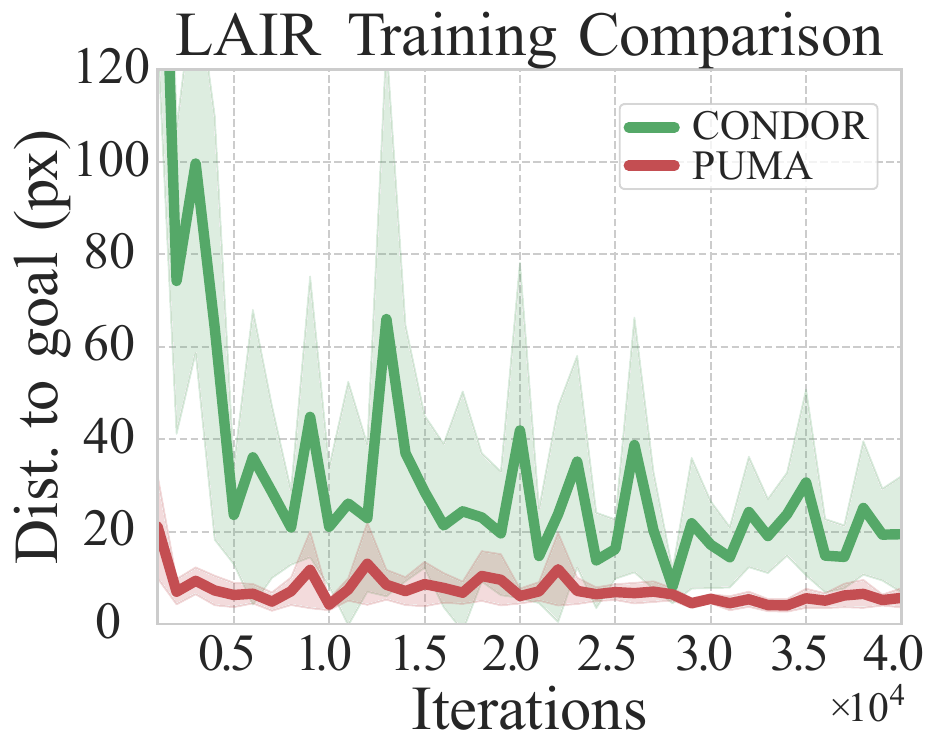}\label{fig:LAIR_plots_b}}
\caption{LAIR dataset PUMA / CONDOR comparison.}
\label{fig:LAIR_plots}
\end{figure}

\paragraph{Stability}
Table \ref{tab:table_lair} presents the results of stability analysis for PUMA on the LAIR dataset. Similar to the findings with the LASA dataset, every variation of PUMA achieves a 0\% rate of unsuccessful trajectories. This validates PUMA's ability to learn stable second-order motions.
\begin{table}[t]
\centering
\caption{Percentage of unsuccessful trajectories over the LAIR dataset ($L$ = 2500, $P$ = 2500, $\epsilon$ = 10px).}
\label{tab:table_lair}
\resizebox{\columnwidth}{!}{%
\begin{tabular}{ccccc}
\hline
\begin{tabular}[c]{@{}c@{}}Behavioral\\ Cloning\end{tabular} & \begin{tabular}[c]{@{}c@{}}PUMA\\ (Euc.)\end{tabular} & \begin{tabular}[c]{@{}c@{}}PUMA\\ (Euc. + $\ell_{\partial}$)\end{tabular} & \begin{tabular}[c]{@{}c@{}}PUMA\\ (Sph.)\end{tabular} & \begin{tabular}[c]{@{}c@{}}PUMA\\ (Euc.+ $\ell_{\partial}$)\end{tabular} \\ \hline
14.1160\%                                                   & \textbf{0.0000\%}                                    & \textbf{0.0000\%}                                                        & \textbf{0.0000\%}                                    & \textbf{0.0000\%}                                                       \\ \hline
\end{tabular}
}
\end{table}

\subsection{Non-Euclidean dataset: LASA $\mathcal{S}^{2}$}
The LASA $\mathcal{S}^{2}$ dataset \citep{urain2022learning} comprises 24 motions, each with 3 demonstrations. Unlike the LASA dataset, the LASA $\mathcal{S}^{2}$ dataset represents these motions in spherical geometry, as indicated by its name, in $\mathcal{S}^{2}$. The sphere, where the motions evolve, is structured similarly to unit quaternions, though with one less dimension. Essentially, we have a 3-dimensional Euclidean space where vectors are constrained to have a unit norm. As a result, the state space is a 2-dimensional spherical manifold embedded in this 3-dimensional space. This setup allows us to examine the performance of PUMA in a manifold with similar attributes to those of unit quaternions. However, the more straightforward visualization of this setup facilitates a more intuitive analysis of PUMA's performance.

\subsubsection{Accuracy}
Fig. \ref{fig:LASA_S2_plots_a} presents the performance of two variations of PUMA, namely when using Euclidean and spherical distance functions. Moreover, BC is included, serving as a lower-bound reference. As explained in Sec.~\ref{sec:exps_lasa}, the boundary loss $\ell_{\partial}$ does not apply in this case and is, therefore, not evaluated. We can observe that the accuracy performance of both variations of PUMA is very similar, and BC performs slightly better than both of them.
\begin{figure}[t]
\centering
    \subfloat[][Comparison of different variations of PUMA.]{\includegraphics[width=0.49\linewidth]{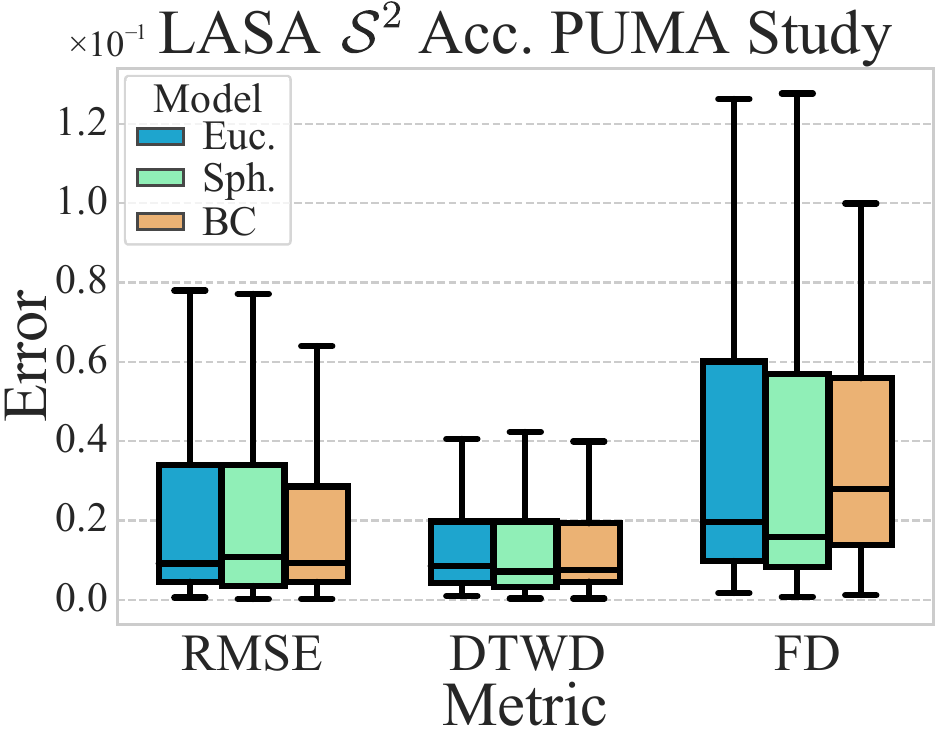}\label{fig:LASA_S2_plots_a}}\hspace{0.03cm}
    \subfloat[][Comparison of PUMA with other state-of-the-art (SoA) methods.]{\includegraphics[width=0.49\linewidth]{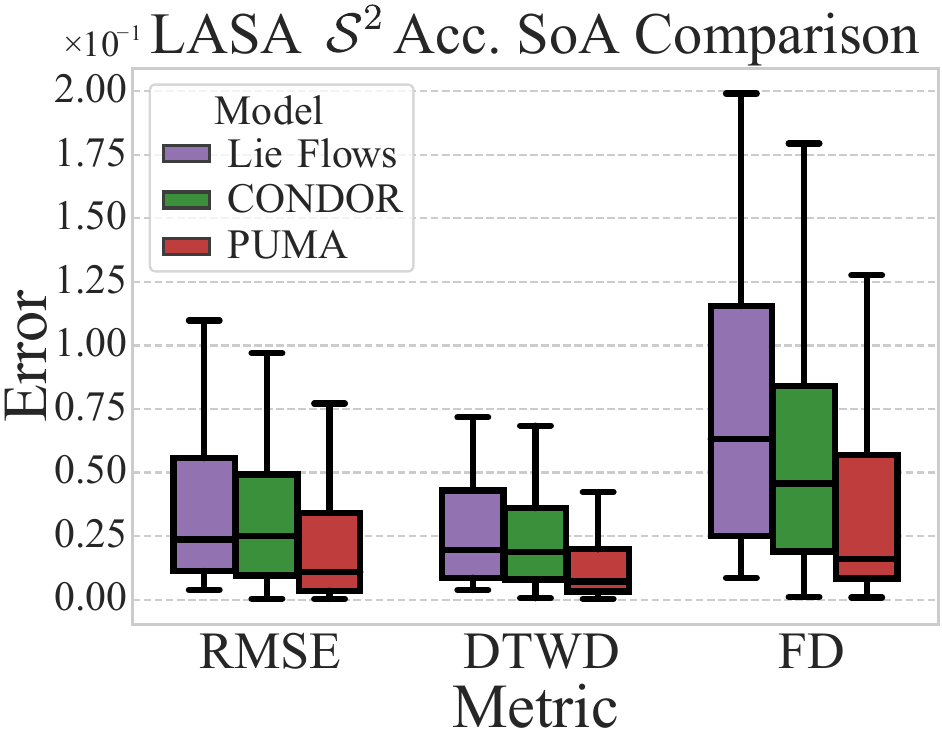}\label{fig:LASA_S2_plots_b}}
\caption{Accuracy study in LASA $\mathcal{S}^{2}$ dataset.}
\label{fig:LASA_S2_plots}
\end{figure}

\subsubsection{Stability}
Table~\ref{tab:table_lasa_s2} depicts the stability results for the LASA $\mathcal{S}^{2}$ dataset. In this case, we also conducted a stability analysis of CONDOR, which, as discussed in Sec.~\ref{sec:related_works}, should not be capable of ensuring stability in non-Euclidean state spaces. This assertion is confirmed by the data in Table~\ref{tab:table_lasa_s2}, which shows that CONDOR yields 7.3133\% of unsuccessful trajectories. When compared to BC, which has a 15.9217\% rate of unsuccessful trajectories, it can be inferred that CONDOR reduces the number of unsuccessful trajectories in non-Euclidean state spaces. However, its performance is still far from satisfactory. In contrast, both versions of PUMA achieve 0\% of unsuccessful trajectories, marking a significant improvement. Moreover, this validates that the Euclidean distance is effective for enforcing stability in spheres, as it can induce spherical metrics, such as the chordal distance, in lower-dimensional manifolds (see Appendix~\ref{learning_spheres}).
\begin{table}[t]
\centering
\caption{Percentage of unsuccessful trajectories over the LASA $\mathcal{S}^{2}$ dataset ($L$ = 2500, $P$ = 2500, $\epsilon$ = 0.06).}
\label{tab:table_lasa_s2}
\begin{tabular}{cccl}
\hline
\begin{tabular}[c]{@{}c@{}}Behavioral\\ Cloning\end{tabular} & CONDOR    & \begin{tabular}[c]{@{}c@{}}PUMA\\ (Euc.)\end{tabular} & \begin{tabular}[c]{@{}l@{}}PUMA\\ (Sph.)\end{tabular} \\ \hline
15.9217\%                                                   & 7.3133\% & \textbf{0.0000\%}                                    & \textbf{0.0000\%}                                    \\ \hline
\end{tabular}
\end{table}

\subsubsection{Qualitative Analysis}
Fig. \ref{fig:LASA_S2_quali} illustrates motions learned with PUMA in $\mathcal{S}^{2}$ using both spherical and Euclidean distances. Within the observed region of the sphere, the vector field exhibits no spurious attractors. Furthermore, when initiated from the same conditions as the demonstrations, the trajectories in both cases show an accurate reproduction of the motions.
\begin{figure}[t]
\centering
    \subfloat[][Results using the great-circle distance.]{\includegraphics[width=0.95\linewidth]{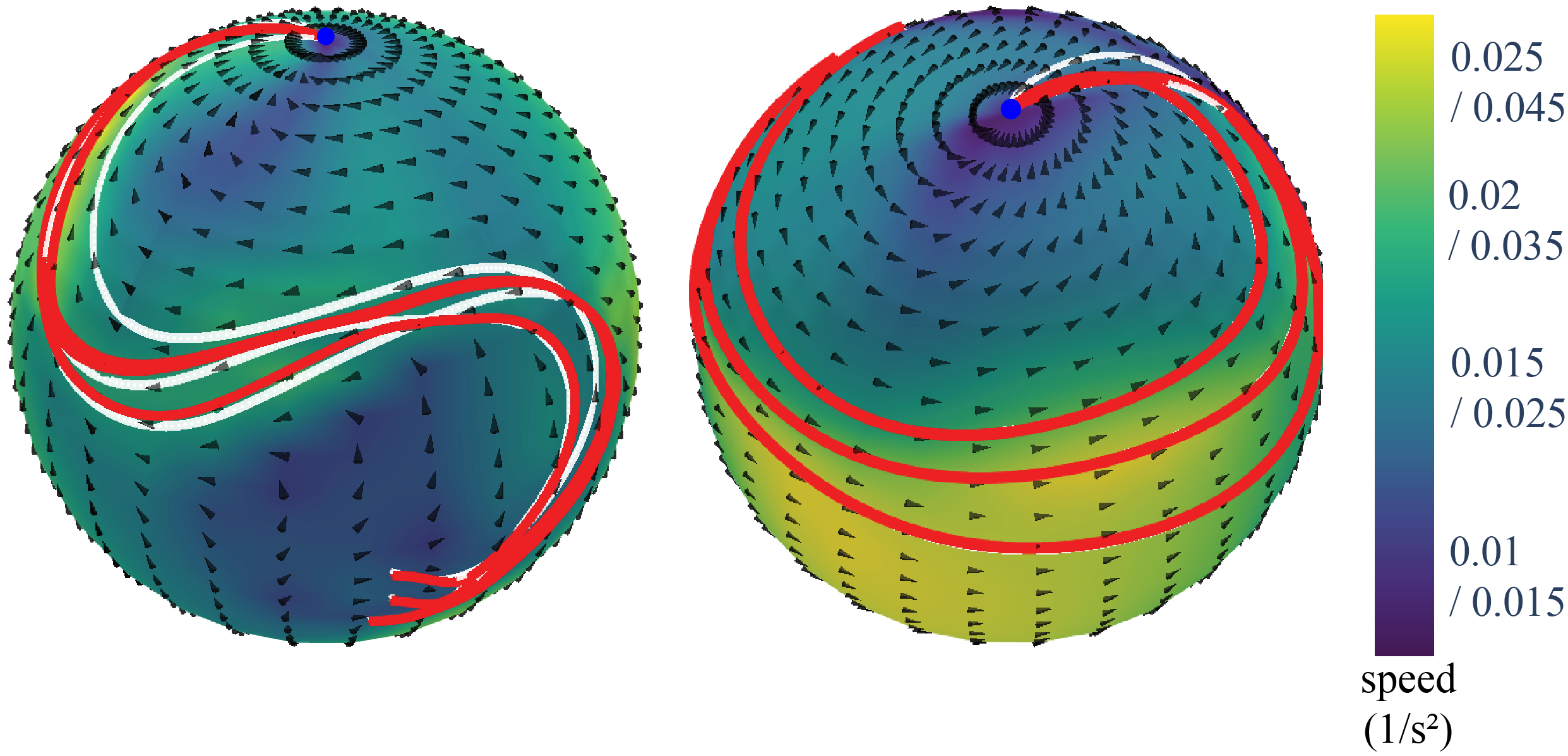}\label{fig:LASA_S2_quali_a}}\\
    \subfloat[][Results using the Euclidean (chordal) distance.]{\includegraphics[width=0.95\linewidth]{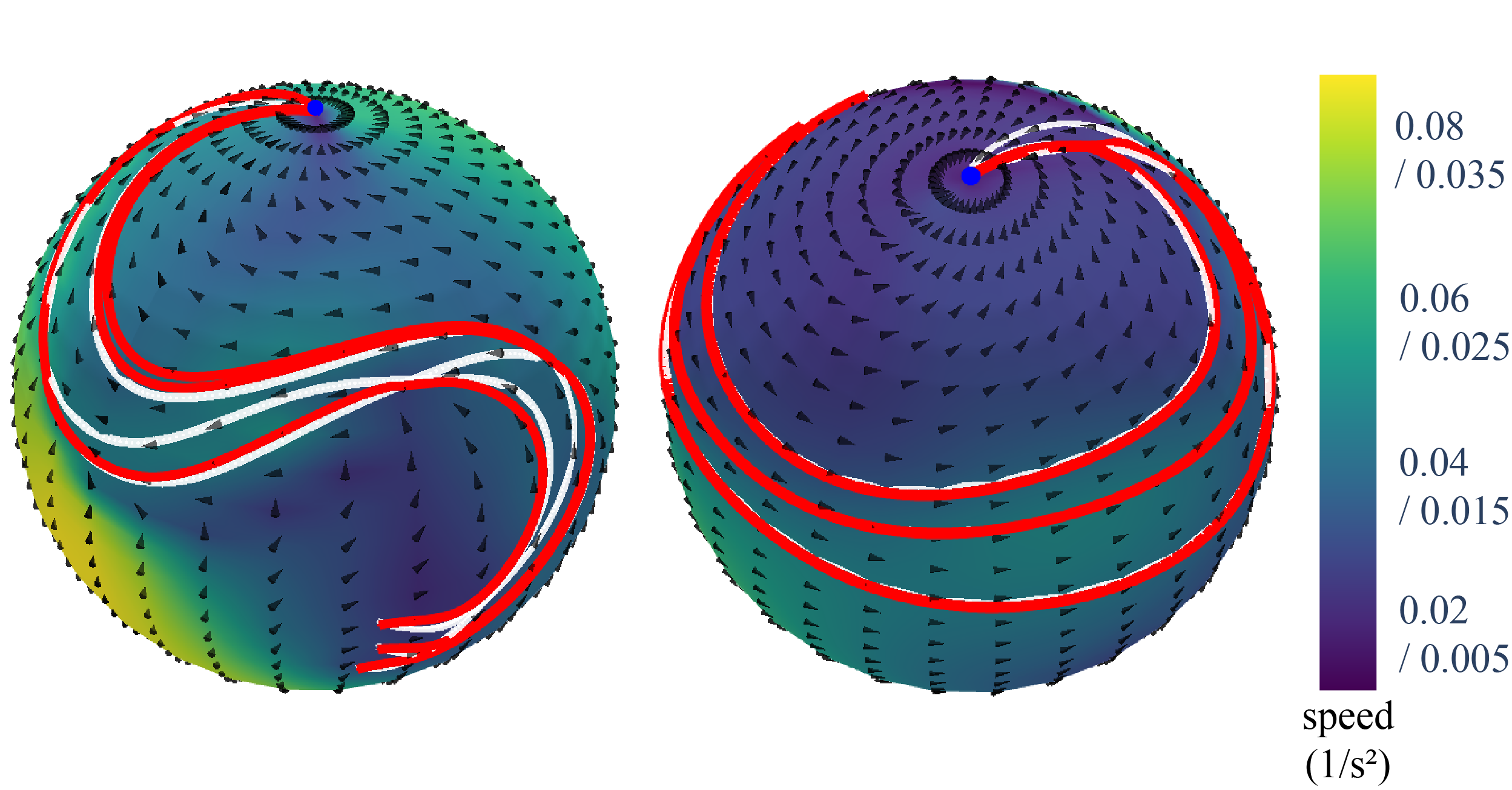}\label{fig:LASA_S2_quali_b}}
\caption{Motions learned with PUMA in $\mathcal{S}^{2}$. White curves represent demonstrations. Red curves represent learned motions when starting from the same initial conditions as the demonstrations. The arrows indicate the vector field of the learned dynamical system.}
\label{fig:LASA_S2_quali}
\end{figure}

\subsubsection{State-of-the-art Comparison}
Fig. \ref{fig:LASA_S2_plots_b} depicts a comparison between three methods: 1) CONDOR, 2) PUMA, and 3) \emph{Lie Flows} \citep{urain2022learning}, a method developed for learning stable motions in non-Euclidean manifolds using tools from Lie Theory. PUMA outperforms both CONDOR and Lie Flows in terms of accuracy. Moreover, in contrast to CONDOR, PUMA also successfully ensured stability in this dataset. 

\subsection{Real-world Experiments}
\label{sec:real_world}
We validate our method in two real-world setups, using two different robots. Both robots have six degrees of freedom, with their end-effector pose represented in $\mathbb{R}^{3} \times \mathcal{S}^{3}$ and are guided using PUMA. In these experiments, $f_{\theta}^{\mathcal{T}}$ is employed to generate velocity references in real-time, which are then sent to the low-level controllers responsible for tracking these references in the robots. Note that throughout the experiments, we operated under the assumption that the target states of the robots were already known.

\begin{figure*}[t]
    \centering
    \includegraphics[width=\linewidth]{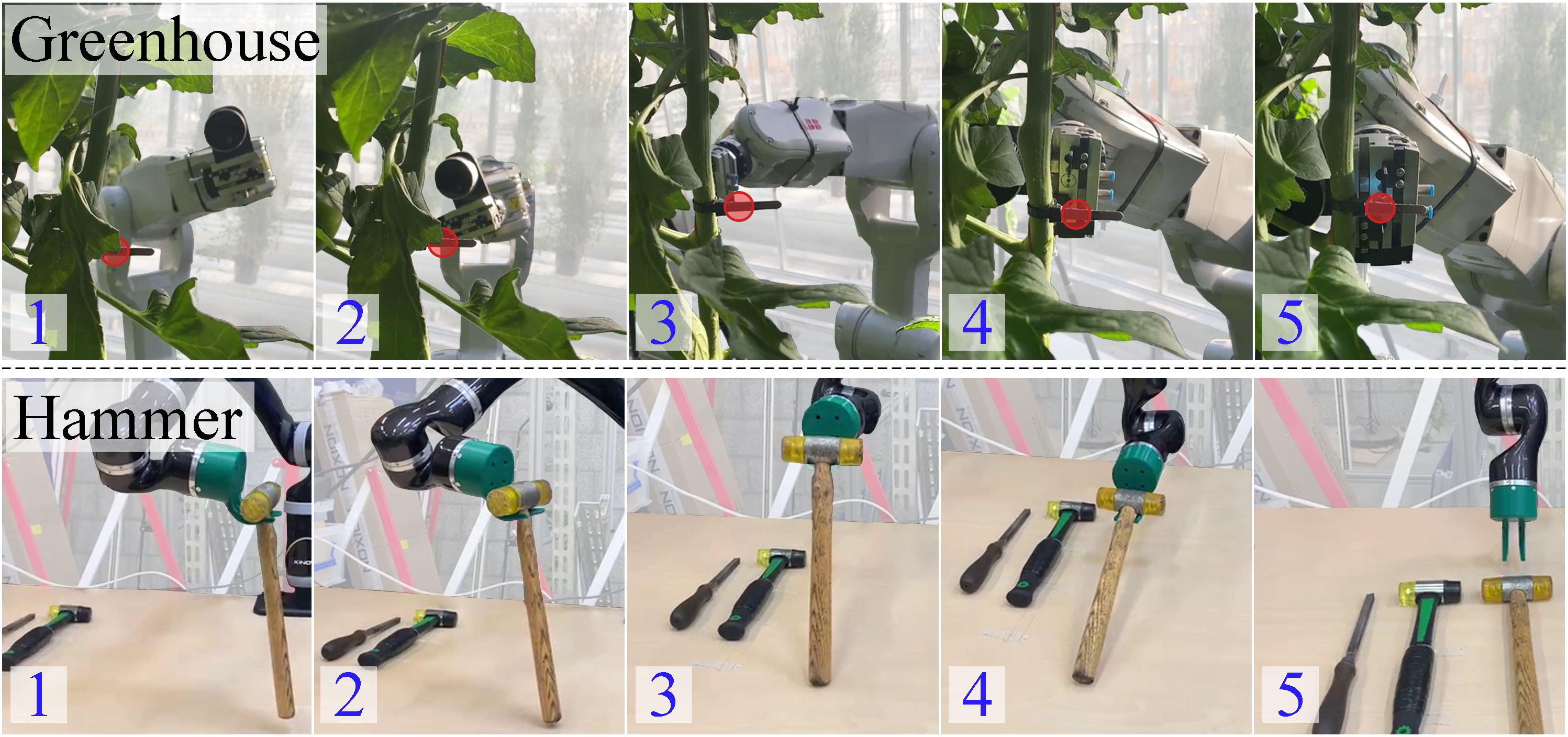}
    \caption{Two sequences of frames, each depicting a robot performing a task: 1) operating in a greenhouse, and 2) placing a hammer. The red circle is used to highlight a target marker in the greenhouse experiment.}
    \label{fig:sequence}
\end{figure*}

\begin{figure}[t]
    \centering
    \includegraphics[width=\linewidth]{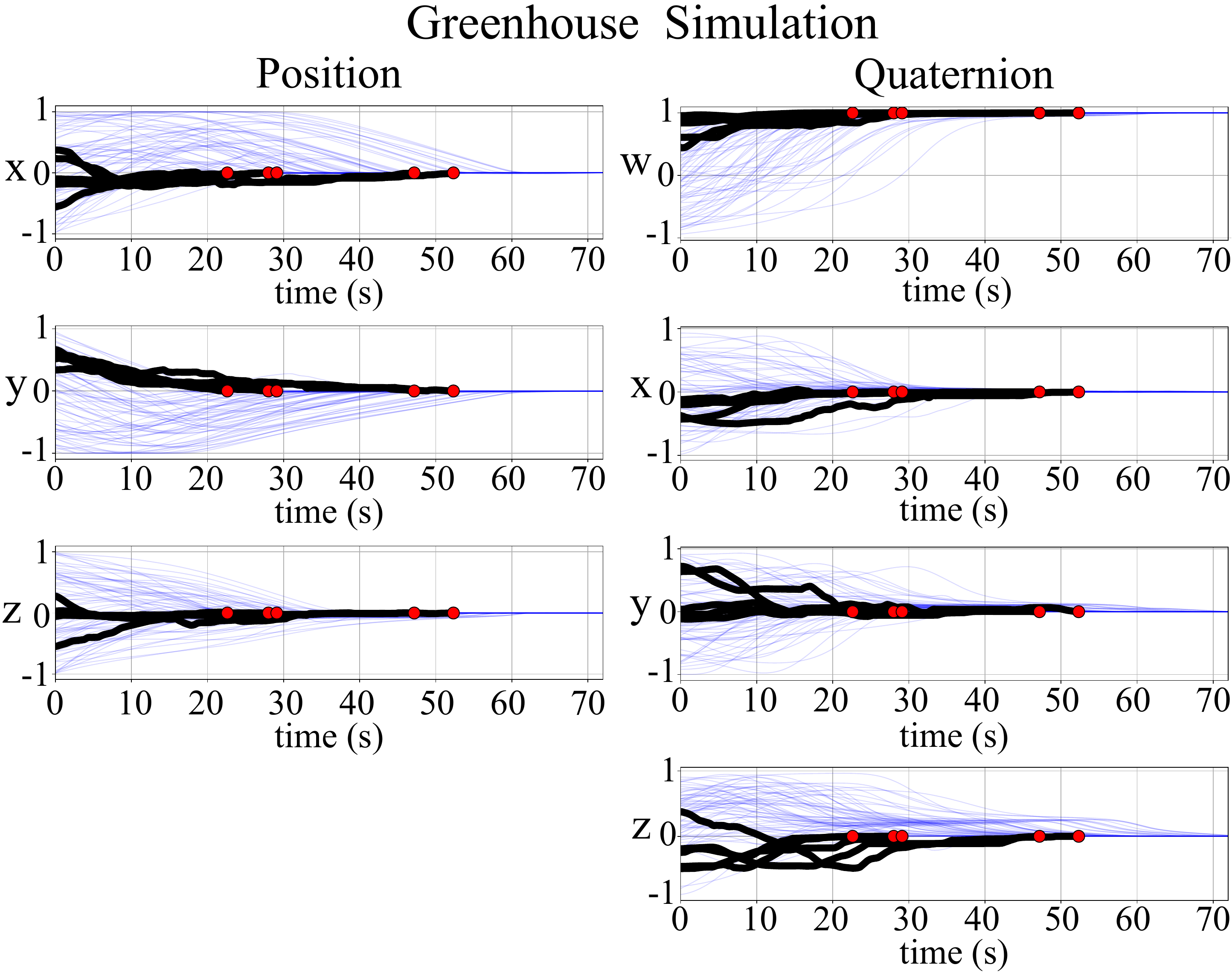}
    \caption{Simulated trajectories, as a function of time, of the greenhouse experiment. Blue trajectories correspond to evaluations of the model under different initial conditions and the black trajectory corresponds to the demonstration. The red point is the goal.}
    \label{fig:greenhouse_simulation}
\end{figure}

\subsubsection{Greenhouse}
This experiment was conducted in a greenhouse, where a robot was trained to reach a black marker on a tomato plant (see Fig. \ref{fig:sequence}). This marker, simulating a plant's peduncle, represented the target state the robot must reach to harvest the plant. Here, the marker allowed us to test the method without causing harm to the plant. This task presents significant challenges, as it requires the robot to perform precise position and orientation control while considering the plant's complex geometry, which is non-trivial to model. In a practical setting, a motion library should be developed to account for the variability across different plants and targets. This library could then be used to select or combine suitable motions to create an appropriate movement strategy considering the plant's characteristics. This experiment represents a preliminary step towards achieving this objective.

We utilized an \emph{ABB IRB 1200} robot equipped with the \emph{Externally Guided Motion}\footnote{Code: \url{https://github.com/ros-industrial/abb_robot_driver}.} module, allowing real-time joint velocity control. As the motion takes place in the end-effector space, the desired velocity at each time step was calculated using $f_{\theta}^{\mathcal{T}}$, and subsequently mapped to the joint space. This conversion was accomplished using the inverse kinematics module with joint limits from the \emph{Robotics Toolbox} \citep{robotics-toolbox}. Similarly, we translated the robot's estimated joint state into the end-effector space using the forward kinematics module from the same toolbox. The same approach was taken to collect the demonstrations of this task directly in the end-effector space, where a space mouse was employed to teleoperate the robot.

Fig. \ref{fig:greenhouse_simulation} presents the demonstrations and simulations of motions conducted in this experiment. It can be observed that five demonstrations were provided. Note that these demonstrations do not have the same end time; only the final state needs to be consistent. Over a span of 70 seconds, all simulated trajectories visibly converge to the goal. For more details, the reader is referred to the attached video.

\subsubsection{Hammer}
The second experiment involves a robot accurately positioning a hammer next to other tools on a table. This task requires precise control of the robot's position and orientation, as shown in Fig. \ref{fig:sequence}. We are interested in controlling the hammer's movement, so we assume that the state of the hammer directly relates to the state of the robot's end effector. This assumption is necessary because the employed low-level controller acts on the robot's end-effector state, and we ensure stability in this state space. Given that the hammer is attached to the robot through a double hook, this assumption is valid as long as the hammer's head is securely held. This assumption breaks down towards the end of the motion when the robot places the hammer on the table. However, since the final states of both the hammer and the robot are similar, we observed that ensuring stability in the end-effector's motion effectively guides the hammer's motion towards its goal state, as can be seen in Fig. \ref{fig:sequence} and in the attached video.

We used the \emph{Kinova Gen2 Ultra lightweight robot} arm with a double hook replacing its default gripper. The control commands obtained from $f_{\theta}^{\mathcal{T}}$ were directly sent to a Cartesian space controller from Kinova\footnote{Kinova's controllers: \url{https://github.com/Kinovarobotics/kinova-ros}}. The demonstrations were collected through kinesthetic teaching, i.e., physically moving the robot along desired paths. The simulated performance in this task was similar to that in the greenhouse, given that the state space was identical in both cases. We refer the reader to the attached video for examples of the robot executing this task from various initial conditions.

%% file: sections/06_conclusion.tex
\section{Conclusions}
We introduced a novel approach for learning stable robotic motions in both Euclidean and non-Euclidean state spaces. To achieve this, we introduced a new loss function based on the triplet loss from the deep metric learning literature. We validated this loss both theoretically and experimentally.

Our approach, PUMA, showcased state-of-the-art performance in every experiment, both in terms of accuracy and stability. It was validated using datasets where the dynamical system evolution was simulated, as well as real robotic platforms where it was employed to provide control commands to low-level controllers. 

Compared to previous work, PUMA offers not only an improvement in addressing non-Euclidean state spaces but also increased flexibility by reducing restrictions in the latent space of the DNN, leading to generally better performance. More specifically, in previous work, the latent dynamical system is constrained to evolve along straight lines towards the goal. In contrast, PUMA allows the latent dynamical system to converge towards a broader range of stable dynamical systems, since its loss function only enforces the latent dynamical system to reduce the distance towards the goal. This feature expands the set of feasible solutions available to the DNN during optimization, thereby enhancing the model's adaptability.

Lastly, while this paper's findings are promising, there are also limitations that can be addressed in future works. First, further exploration of the scalability properties of PUMA is needed, as the largest DNN input employed in this paper had 7 dimensions. Second, we have so far focused on learning independent motion primitives for specific tasks. A relevant line of research would be integrating this model into a larger framework where multiple primitives are learned and combined together. Third, a topic that we did not address in this work is obstacle avoidance. Recent techniques, such as Geometric Fabrics \cite{van2022geometric}, exploit dynamical systems to represent motions that achieve full-body obstacle avoidance, which can be integrated with dynamical systems learned in the end-effector space of the robot. Consequently, an exciting avenue for future work is exploring the combination of these methods with PUMA. Finally, another interesting research direction is studying the integration of these approaches with the low-level control of the robots. Currently, it is assumed that the transitions requested by PUMA can be tracked by the controllers of the robot. However, this is not always the case, and it would be beneficial to incorporate this information into the learning framework.

%% file: sections/appendix/appendix.tex
\section*{Appendix}
\appendix
\input{sections/appendix/A}
\input{sections/appendix/B}
\input{sections/appendix/C}

%% file: sections/appendix/A.tex
\section{Upper Bound $\beta$}
\label{app:upperbound}
In this section, we introduce and prove the proposition used in Sec.~\ref{sec:surrogate_stability} to demonstrate that the surrogate stability conditions ensure asymptotic stability in the dynamical system $f_{\theta}^{\mathcal{T}}(x_{t})$.

\begin{proposition}[Existence of class-$\mathcal{KL}$ function]
\label{prop:existence_KL}
Consider a dynamical system ${\dot{y}_{t} = f(y_{t})}$ with ${f:\mathcal{L} \subset \mathbb{R}^{n} \to \mathbb{R}^{n}}$ continuously differentiable, and ${\Phi^{y}_{\theta}(t,y_{0}):\mathbb{R}_{\geq 0} \times \mathcal{L} \to \mathcal{L}}$ as its evolution function. For a distance function $d_{t}$ in $\mathcal{L}$, we define its evolution, for a given $t$ and $y_{0}$, as $\delta:\mathcal{L} \times \mathbb{R}_{\geq 0} \to \mathbb{R}_{\geq 0}$, with $\delta(y_{0}, t)=||y_{\mathrm{g}} - \Phi^{y}_{\theta}(t,y_{0})||$.

Then, consider the function ${\beta:\mathbb{R}_{\geq 0} \times \mathbb{R}_{\geq 0} \to \mathbb{R}_{\geq 0}}$ defined as
\begin{equation}
\label{eq:beta_app}
    \beta(d_{0}, t) = z_{0} + \int_{0}^{t}\dot{z}_{s} ds,
\end{equation}
with derivative
\begin{equation}
\label{eq:dot_z_app}
    \dot{z}_{t} = \alpha \bigl(z_{t} - \delta^{\text{max}}(d_{0}, t)\bigr),
\end{equation}
where $\alpha \in \mathbb{R}_{<0}$ and 
\begin{equation}
\label{eq:delta_max}
    \delta^{\text{max}}(d_{0}, t) = \max_{y_{0} \in \mathcal{Y}_{0}} \left( \max_{s \in [t, t+\Delta t]} \delta(y_{0}, s) \right),
\end{equation}
with $\mathcal{Y}_{0}(d_{0}) = \{y_{0} \in \mathcal{L} : ||y_{\mathrm{g}} - y_{0}|| = d_{0}\}$ and ${\Delta t \in \mathbb{R}_{>0}}$. 
The initial condition $z_{0}$ is set as
\begin{equation}
\label{eq:z0_app}
    z_{0} =\delta^{\text{max}}_{0} +  d_{0},
\end{equation}
where $\delta^{\text{max}}_{0} = \delta^{\text{max}}(d_{0}, 0)$.

Then, under the conditions of Theorem~\ref{theo:puma_surrogate}, i.e., $\forall t \in \mathbb{R}_{\geq0}$,
\begin{enumerate}
    \item $d_{t} = d_{t + \Delta t}$, \hspace{0.6cm} for $y_{0} =y_{\mathrm{g}}$,
    \item $d_{t} > d_{t+\Delta t}$, \hspace{0.5cm} $\forall y_{0} \in \mathcal{L} \setminus \{y_{\mathrm{g}}\}$,
\end{enumerate}
there exists a class-$\mathcal{KL}$ function $\beta$ such that $\forall y_{t} \in \mathcal{L}$ and $\forall t \in \mathbb{R}_{\geq 0}$,
\begin{equation}
    ||y_{\mathrm{g}} - y_{t}|| \leq \beta(d_{0}, t),
\end{equation}
which can also be expressed as 
\begin{equation}
    \delta(y_{0}, t) \leq \beta(d_{0}, t).
\end{equation}
\end{proposition}

\begin{proof}
To prove that a function $\beta$ is an upper bound of $\delta$ and is class-$\mathcal{KL}$, we need the following conditions to hold $\forall y_{t} \in \mathcal{L}$ and $\forall t \in \mathbb{R}_{\geq 0}$:
\begin{enumerate}[label=\roman*)]
    \item $\beta(0, t) = 0$,
    \item $\beta$ decreases with $t$,
    \item $\delta \leq \beta$,
    \item $\beta$ is continuous with respect to $d_{0}$ and $t$,
    \item $\beta \to 0$ as $t \to \infty$,
    \item $\beta$ strictly increases with $d_{0}$.
\end{enumerate}
Therefore, we proceed to prove all of these conditions. A summary of this proof is depicted in Fig. \ref{fig:flow_proof}.
\begin{figure}[t]
    \centering
    \includegraphics[width=\columnwidth]{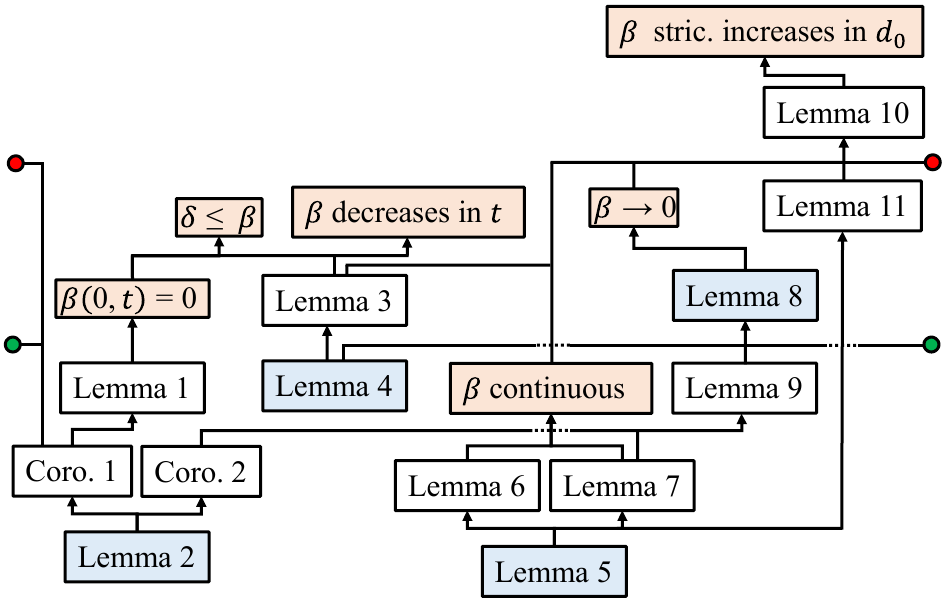}
    \caption{Summary of lemmas and corollaries related to the proof of Proposition~\ref{prop:existence_KL}. Bisque-colored boxes indicate the conditions necessary for the proposition to be true, while blue boxes represent lemmas relying directly on its assumptions and requirements.}
    \label{fig:flow_proof}
\end{figure}

\vspace{0.2cm}
\paragraph{$\beta(0, t) = 0$}
To demonstrate that this condition is valid, we introduce Lemma \ref{lemma:beta_at_0}. For this lemma to be applicable, a specific condition must be met. To validate this condition, we refer to Lemma \ref{lemma:constant_delta} and Corollary~\ref{coro:d_max_0}, both of which will be introduced subsequently. We also present Corollary~\ref{coro:mind_0}, which will prove beneficial in subsequent discussions.

\begin{lemma}[$\beta(0, t) = 0$ constant]
\label{lemma:beta_at_0}
$\beta(0, t) = 0$ if ${\delta^{\text{max}}(0, t)=0}$, $\forall t \in \mathbb{R}_{\geq0}$.
\end{lemma}
\begin{proof}
    If ${\delta^{\text{max}}(0, t)=0}$, $\forall t \in \mathbb{R}_{\geq0}$, for $d_{0} = 0$, from equations \eqref{eq:z0_app} and \eqref{eq:dot_z_app}, we can observe that $z_{0} = 0$ and $\dot{z}_{t}=0$, $\forall t \in \mathbb{R}_{\geq0}$. Replacing this in \eqref{eq:beta_app}, we conclude that $\beta(0, t) = 0$, $\forall t \in \mathbb{R}_{\geq 0}$. 
\end{proof}

\begin{lemma}[$\delta = 0$ constant]
\label{lemma:constant_delta}
${\delta(y_{\mathrm{g}},t)=0}$, $\forall t \in \mathbb{R}_{\geq0}$, if:
\begin{enumerate}[font=\itshape]
\item $d_{t} = d_{t + \Delta t}$, \hspace{0.2cm} for $y_{t+\Delta t}=y_{\mathrm{g}}$ (Theorem~\ref{theo:puma_surrogate}, 1),
\item $d_{t} > d_{t+\Delta t}$, \hspace{0.2cm} $\forall y_{t} \in \mathcal{L} \setminus \{y_{\mathrm{g}}\}$ (Theorem~\ref{theo:puma_surrogate}, 2).
\end{enumerate}
\end{lemma}
\begin{proof}
Let us introduce $t_{\mathrm{g}}$, which is the earliest $t \in \mathbb{R}_{\geq 0}$ where $y_{t} = y_{\mathrm{g}}$. Then, condition 1 gives rise to two scenarios for every state visited after $t_{\mathrm{g}}$:
\begin{enumerate}[label=\roman*)]
	\item \emph{$\delta(y_{0}, t)$ is constant:} A constant function satisfies $d_{t} = d_{t + \Delta t}$.
	\item \emph{$\delta(y_{0}, t)$ is periodic:} A set of functions could meet the condition $d_{t} = d_{t + \Delta t}$ by varying over time but always returning to the same value after $\Delta t$ seconds. However, in this case, the only possibility is that, for all $t > t_{\mathrm{g}}$, we have $\delta(y_{0}, t) = \delta(y_{0}, t + \Delta t)$, i.e., a periodic function. This is because the time derivative of $\delta$ is solely defined by $y_{t}$, since $\partial \delta / \partial t = \partial \delta / \partial y_{t} \cdot \dot{y}_{t}$, and both of these variables only depend on $y_{t}$. Therefore, $\partial \delta / \partial t$ at $y_{\mathrm{g}}$, and at any state visited afterwards, given the state, must always be the same. This implies that $\delta$ must evolve over time to the same states as those to which it evolved $\Delta t$ seconds ago, since for every state, the derivative is the same as it was $\Delta t$ seconds before.
\end{enumerate}

However, the periodic case for $\delta(y_{0}, t)$ contradicts condition 2, which requires that $\delta$ strictly decreases after $\Delta t$ seconds, and in the periodic scenario, it returns to the same value after $\Delta t$ seconds. Hence, the only remaining scenario is that $\delta$ is constant after $t_{\mathrm{g}}$, with a value of 0 (by definition). Lastly, by noting that $y_{0} = y_{\mathrm{g}}$ implies that $t_{\mathrm{g}}=0$, we get that ${\delta(y_{\mathrm{g}},t)=0}$, $\forall t \in \mathbb{R}_{\geq0}$. 
\end{proof}

\begin{corollary}[$\delta^{\text{max}}(0, t)=0$]
\label{coro:d_max_0}
${\delta^{\text{max}}(0, t)=0}$ if ${\delta(y_{\mathrm{g}},t)=0}$, $\forall t \in \mathbb{R}_{\geq0}$.
\end{corollary}
\begin{proof}
If $\delta(y_{\mathrm{g}}, t) = 0$, $\forall t \in \mathbb{R}_{\geq0}$, and noting $d_{0}=0$ implies that every $y_{0} \in \mathcal{Y}_{0}$ in \eqref{eq:delta_max} is equal to $y_{\mathrm{g}}$, from \eqref{eq:delta_max}, it follows that $\delta^{\text{max}}(0, t) = 0$, $\forall t \in \mathbb{R}_{\geq0}$. 
\end{proof}

\begin{corollary}[$\min(\delta^{\text{max}})$]
\label{coro:mind_0}
$\min(\delta^{\text{max}}) = 0$ if ${\delta^{\text{max}}(0, t)=0}$, $\forall t \in \mathbb{R}_{\geq0}$.
\end{corollary}
\begin{proof}
Recall that $\delta$ is positive definite and that $\delta^{\text{max}}$ takes the value of some $\delta$. Then, ${\delta^{\text{max}}(0, t)=0}$, $\forall t \in \mathbb{R}_{\geq0}$, indicates that this is the minimum value $\delta^{\text{max}}$ can achieve. Hence, ${\min(\delta^{\text{max}})=0}$. 
\end{proof}

Then, provided that the conditions of Proposition~\ref{prop:existence_KL} are met, the conditions of Lemma \ref{lemma:constant_delta} are fulfilled. Therefore, $\delta(y_{\mathrm{g}}, t) = 0$ for every $t \in \mathbb{R}_{\geq 0}$, and, hence, ${\delta^{\text{max}}(0, t)=0}$ (from Corollary~\ref{coro:d_max_0}). Consequently, we can employ Lemma \ref{lemma:beta_at_0} to demonstrate that $\beta(0, t) = 0$ for all $t \in \mathbb{R}_{\geq 0}$.

\vspace{0.2cm}
\paragraph{$\beta$ decreases with $t$, and $\delta \leq \beta$}
We introduce Lemma \ref{lemma:upper_decreasing}, which presents one condition that, if satisfied, implies that $\beta > \delta$ and $\beta$ strictly decreases over time, $\forall d_{0} \in \mathbb{R}_{\geq0} \setminus \{ 0\}$. After proving this lemma, we introduce Lemma \ref{lemma:maxd_decreases}, which is then used to show that this condition is satisfied in our case. Lastly, we combine this with the fact that $\beta(0, t) = 0$ to conclude that $\beta$ decreases with $t$ (i.e., is a non-increasing function), and $\delta \leq \beta$, $\forall d_{0} \in \mathbb{R}_{\geq0}$.

\begin{lemma}[$\beta > \delta$  and strictly decreasing]
    \label{lemma:upper_decreasing}
        $\beta(d_{0}, t) > \delta(y_{0}, t)$ and $\beta(d_{0}, t)$ strictly decreases with respect to $t$, $\forall d_{0} \in {\mathbb{R}_{\geq 0} \setminus \{0\}}$, $\forall t \in \mathbb{R}_{\geq0}$, if:
    \begin{enumerate}
        \item $\delta^{\text{max}}$ decreases with respect to $t$, $\forall d_{0} \in \mathbb{R}_{\geq 0} \setminus \{0\}$.
    \end{enumerate}
\end{lemma}
\begin{proof}
Although $\delta^{\text{max}}$ evolves as a function of time, we can infer from \eqref{eq:dot_z_app} that, locally, $\beta$ behaves as a linear first-order dynamical system, with the origin at $\delta^{\text{max}}$. Given that $\alpha < 0$, it follows that $\beta$ will converge towards $\delta^{\text{max}}$ over time.

Moreover, \eqref{eq:z0_app} indicates that $z_{0} > \delta^{\text{max}}_{0}$ for every $d_{0} > 0$. Given the properties of linear first-order systems, $\beta$ will strictly decrease towards $\delta^{\text{max}}_{0}$, without ever reaching or overshooting this value. Combined with condition 1, which indicates that $\delta^{\text{max}}$ decreases with respect to $t$ for all $d_{0} \in \mathbb{R}_{\geq 0} \setminus \{0\}$, we can conclude that $\beta$ will always be greater than $\delta^{\text{max}}$ and will continue to strictly decrease towards this value as a function of time. Therefore, $\beta > \delta$ and  $\beta$ strictly decreases with respect to $t$, $\forall d_{0} \in \mathbb{R}_{\geq 0} \setminus \{0\}$. 

\end{proof}

\begin{lemma}[$\delta^{\text{max}}$ decreases over time]
\label{lemma:maxd_decreases}
    $\delta^{\text{max}}(d_{0}, t)$ decreases with respect to $t$, $\forall d_{0} \in \mathbb{R}_{\geq 0} \setminus \{0\}$, $\forall t \in \mathbb{R}_{\geq0}$, if:
    \begin{enumerate}
        \item $d_{t}>d_{t+\Delta t}$, $\forall t \in \mathbb{R}_{\geq0}$, $\forall y_{0} \in \mathcal{L} \setminus \{y_{\mathrm{g}}\}$, (Theorem~\ref{theo:puma_surrogate}, 2).
    \end{enumerate}
\end{lemma}
\begin{proof}
Let us define
\begin{equation}
\label{eq:delta_y0}
    \delta_{t + \Delta t}^{\text{max}}(y_{0}, t)=\max_{s \in [t, t+\Delta t]} \delta(y_{0}, s),
\end{equation}
which allows us to write $\delta^{\text{max}} = \max_{y_{0} \in \mathcal{Y}_{0}}\left(\delta_{t + \Delta t}^{\text{max}}\right)$. We will first prove that $\delta_{t + \Delta t}^{\text{max}}$ decreases with $t$. 

To do so, observe that condition 1 can be rewritten as $\delta(y_{0}, t + \Delta t) < \delta(y_{0}, t)$. Given that this condition holds true for every $s$ in the interval $[t, t + \Delta t]$, none of the values of $\delta(y_{0}, s)$ computed in this interval can exceed the maximum of this interval, i.e., $\delta_{t + \Delta t}^{\text{max}}(y_{0}, t)$, after $\Delta t$ seconds. If there were such a value $\delta(y_{0}, s)$ that is greater than $\delta_{t + \Delta t}^{\text{max}}(y_{0}, t)$ after $\Delta t$ seconds, it would imply that $\delta(y_{0}, s)$ had increased, for some $s$, after $\Delta t$, contradicting condition 1. Therefore, $\delta_{t + \Delta t}^{\text{max}}$ can only decrease as time progresses. Since this holds for all instances of $t$, we conclude that $\delta_{t + \Delta t}^{\text{max}}$ decreases with respect to $t$, $\forall y_{0} \in \mathcal{L} \setminus \{y_{\mathrm{g}}\}$.

Now, the function $\delta^{\text{max}}$ computes the maximum over a set of functions $\delta_{t + \Delta t}^{\text{max}}$ with different initial conditions $y_{0} \in \mathcal{Y}_{0}(d_{0})$. However, for any given $d_{0}$, $\forall y_{0} \in \mathcal{L} \setminus \{y_{\mathrm{g}}\}$, each one of these functions decreases with respect to time. Then, for any given interval of time, we have two possible scenarios:
\begin{enumerate}[label=\roman*)]
    \item $\delta^{\text{max}}$ is equal to one function $\delta_{t + \Delta t}^{\text{max}}$ (the current maximum). In this case, $\delta^{\text{max}}$ decreases with time, since $\delta_{t + \Delta t}^{\text{max}}$ decreases with time.
    \item The function that is currently the maximum over the set of functions $\delta_{t + \Delta t}^{\text{max}}$ reaches a point in time $t'$ where it changes. At the point where the change occurs, the function will be lower than or equal to its previous values for $t<t'$, since the previous maximum is decreasing. Furthermore, it will be greater than or equal to the values for $t>t'$, since the new maximum also decreases.
\end{enumerate}
Therefore, we can conclude that $\forall d_{0}$ where $y_{0} \in \mathcal{Y}_{0}(d_{0})$ is not $y_{\mathrm{g}}$, $\delta^{\text{max}}$ decreases with respect to time. Moreover, by noting that $\forall y_{0} \in \mathcal{Y}_{0}(d_{0})$, $d_{0} \neq 0$ implies that $y_{0} \neq y_{\mathrm{g}}$, it follows that $\delta^{\text{max}}$ decreases with respect to time, $\forall d_{0} \in \mathbb{R}_{\geq 0} \setminus \{0\}$, $\forall t \in \mathbb{R}_{\geq0}$. 
\end{proof}

As a consequence, the assumptions and conditions outlined in Proposition~\ref{prop:existence_KL} enable us to apply Lemma \ref{lemma:maxd_decreases} to demonstrate that the conditions of Lemma \ref{lemma:upper_decreasing} are satisfied. This, in turn, leads us to the conclusion that $\beta > \delta$ and $\beta(d_{0}, t)$ strictly decreases with respect to $t$, $\forall d_{0} \in {\mathbb{R}_{\geq 0} \setminus \{0\}}$, $\forall t \in \mathbb{R}_{\geq0}$. Lastly, before we showed that, for $d_{0}=0$, $\beta(0, t) = 0$, $\forall t \in \mathbb{R}_{\geq0}$. Therefore, in general, we have that $\beta \geq \delta$ and $\beta$ decreases with respect to $t$, $\forall d_{0} \in {\mathbb{R}_{\geq 0}}$, $\forall t \in \mathbb{R}_{\geq0}$.

\vspace{0.2cm}
\paragraph{Continuity of $\beta$}
We require $\beta$ to be continuous with respect to $t$, for each fixed $d_{0}$, and with respect to $d_{0}$, for each fixed $t$. To achieve this, we will introduce Lemmas \ref{lemma:d_continuous}, \ref{lemma:dmax_continuous_t} and \ref{lemma:dmax_continuous_d}.

\begin{lemma}
\label{lemma:d_continuous}[Continuity of $\delta$]
If $\dot{y}_{t}=f(y_{t})$ is continuously differentiable, then $\delta(y_{0}, t)$ is continuous with respect to $y_{0}$ and $t$.
\end{lemma}
\begin{proof}
Given that $\dot{y}_{t} = f(y_{t})$ is continuously differentiable, it follows that $\Phi^{y}_{\theta}(t, y_{0})$ is continuously differentiable with respect to both $t$ and $y_{0}$ \citep{arnold1992ordinary}. Moreover, since $\delta$ is a metric on $\mathcal{L}$ defined by $\delta = ||y_{\mathrm{g}} - \Phi^{y}_{\theta}(t,y_{0})||$, it is continuous with respect to $\Phi^{y}_{\theta}(t,y_{0})$ \citep{munkres2000topology}. Since $\Phi^{y}_{\theta}(t,y_{0})$ is continuously differentiable and thus continuous, and the composition of two continuous functions is continuous, it follows that $\delta(y_{0}, t)$ is continuous with respect to both $y_{0}$ and $t$. 
\end{proof}

\begin{lemma}[$\delta^{\text{max}}$ continuous with respect to $t$]
\label{lemma:dmax_continuous_t}
$\delta^{\text{max}}$ is a continuous function with respect to $t$, if $\delta$ is continuous with respect to $t$.
\end{lemma}
\begin{proof}
We will first prove this for $\delta_{t + \Delta t}^{\text{max}}$ (as defined in \eqref{eq:delta_y0}). Given that $\delta_{t + \Delta t}^{\text{max}}$ is the maximum value of $\delta(y_{0}, s)$ over the interval $[t, t + \Delta t]$, any change in $\delta_{t + \Delta t}^{\text{max}}$ must be due to a change in $\delta$ for some $s$ in the boundaries of $[t, t + \Delta t]$. Then, since $\delta$ changes continuously, $\delta_{t + \Delta t}^{\text{max}}$ can only change continuously as well. Thus, $\delta_{t + \Delta t}^{\text{max}}$ is continuous with respect to $t$.

However, we need to show that $\delta^{\text{max}} = \max_{y_{0} \in \mathcal{Y}_{0}}\left(\delta_{t + \Delta t}^{\text{max}}\right)$ is continuous with respect to $t$. This follows from the fact that $\delta^{\text{max}}$ computes the point-wise maximum along $t$ over a set of continuous functions $\delta_{t + \Delta t}^{\text{max}}$, which results in a continuous function \citep{pugh2015real}. 
\end{proof}

\begin{lemma}[$\delta^{\text{max}}$ continuous with respect to $d_{0}$]
\label{lemma:dmax_continuous_d}
$\delta^{\text{max}}$ is continuous with respect to $d_{0}$, if $\delta$ is continuous with respect to $y_{0}$.
\end{lemma}
\begin{proof}
Given the continuity of the function $\delta(y_{0}, t)$ with respect to $y_{0}$, the function $\delta_{t + \Delta t}^{\text{max}}$ computes the point-wise maximum along $y_{0}$ over a set of continuous functions, specifically $\{\delta(y_{0}, s): s \in [t, t + \Delta t]\}$. Consequently, $\delta_{t + \Delta t}^{\text{max}}$ is also continuous in $y_{0}$ \citep{pugh2015real}.

Building on this, we seek to demonstrate that $\delta^{\text{max}}=\max_{y_{0} \in \mathcal{Y}_{0}}\left(\delta_{t + \Delta t}^{\text{max}}\right)$ is continuous with respect to $d_{0}$. To accomplish this, we will use \emph{the maximum theorem} \citep{border1985fixed}. This theorem states that if $\mathcal{Y}_{0}(d_{0})$ is \emph{continuous} with respect to $d_{0}$ and \emph{compact-valued}\footnote{The concepts of \emph{compact-valued} and \emph{continuous} refer to those employed in the literature of set-valued functions/correspondences. For further details, we refer the reader to \citep{border1985fixed}.}, and $\delta_{t + \Delta t}^{\text{max}}(y_{0}, t)$ is continuous in $y_{0}$, then $\delta^{\text{max}}(d_{0}, t)$ is continuous in $d_{0}$.

$\mathcal{Y}_{0}(d_{0})$ is continuous because the relationship between each $y_{0}$ and $d_{0}$, i.e., the metric on $\mathcal{L}$, is continuous \citep{munkres2000topology}. $\mathcal{Y}_{0}(d_{0})$ is compact-valued if, for each $d_{0}$, $\mathcal{Y}_{0}$ constitutes a compact set. Given that $\mathcal{Y}_{0} \subset \mathbb{R}^{n}$, the \emph{Heine-Borel theorem} \citep{pugh2015real} indicates that $\mathcal{Y}_{0}$ is compact if it is both closed and bounded. Every $\mathcal{Y}_{0}$ is closed as it fulfills an algebraic equation, and can be contained within any ball of radius larger than $d_{0}$, making it bounded. Consequently, $\mathcal{Y}_{0}(d_{0})$ is compact-valued.

Therefore, since $\delta_{t + \Delta t}^{\text{max}}(y_{0}, t)$ is continuous with respect to $y_{0}$, and $\mathcal{Y}_{0}(d_{0})$ is both continuous and compact-valued, the maximum theorem states that $\delta^{\text{max}}(d_{0}, t)$ is continuous with respect to $d_{0}$. 
\end{proof}

Now, we can proceed to conclude about the continuity of $\beta$ for both $t$ and $d_{0}$.
\begin{enumerate}[leftmargin=*]
    \item \emph{Continuity in $t$}: For any given $d_{0}$, since $\beta$ evolves as a linear first-order system (as per \eqref{eq:beta_app}), its solution will exist provided that $\delta^{\text{max}}$ is continuous with respect to $t$ \citep{nagle2018fundamentals}. Then, under the assumptions of Proposition~\ref{prop:existence_KL}, Lemma \ref{lemma:dmax_continuous_t} confirms that $\delta^{\text{max}}$ is indeed continuous, provided that $\delta$ is continuous. From Lemma \ref{lemma:d_continuous} we infer that $\delta$ is continuous with respect to $t$. Hence, $\beta$ is well-defined for any given $d_{0}$,  indicating that it is differentiable with respect to $t$, and, therefore, continuous.
    \item \emph{Continuity in $d_{0}$}: For any given $t$, from \eqref{eq:beta_app}, we know that $\beta$ will be continuous with respect to $d_{0}$ as long as both of its terms, $z_{0}$ and the intregral of $\dot{z}_{t}$, are continuous. The continuity of both terms depends on the continuity of $\delta^{\text{max}}$. Then, since Lemma \ref{lemma:dmax_continuous_d} indicates that $\delta^{\text{max}}$ is continuous with respect to $d_{0}$ provided that $\delta$ is continuous with respect to $y_{0}$, and Lemma \ref{lemma:d_continuous} confirms that this is indeed the case (given the assumptions of Proposition~\ref{prop:existence_KL}), we conclude that $\beta$ is continuous with respect to $d_{0}$.
\end{enumerate}

\vspace{0.2cm}
\paragraph{$\beta\to 0$ as $t \to \infty$}
To prove that this condition holds, we will utilize Lemma \ref{lemma:beta_to_zero}. The application of this lemma necessitates the use of Lemma \ref{lemma:maxd_decreases} and Corollary~\ref{coro:d_max_0},  which are already introduced previously, and Lemma \ref{lemma:dmax_d0}, which is introduced afterwards.

\begin{lemma}[$\beta$ converges to zero]
\label{lemma:beta_to_zero}
$\beta\to 0$ as $t \to \infty$ if:
\begin{enumerate}
\item $\delta^{\text{max}} \in [0, \max(\delta)]$,
\item $\delta^{\text{max}}$ decreases with respect to $t$, $\forall d_{0} \in \mathbb{R}_{\geq0}$, $\forall t \in \mathbb{R}_{\geq 0}$,
\item $d_{t}<d_{t-\Delta t}$, $\forall y_{0} \in \mathcal{L} \setminus \{y_{\mathrm{g}}\}$, $\forall t \in \mathbb{R}_{\geq 0}$, (Theorem~\ref{theo:puma_surrogate}, 2).
\item $\delta^{\text{max}}$ surjective in $d_{0}$,
\item $\delta^{\text{max}}(0, t)=0$, $\forall t \in \mathbb{R}_{\geq0}$.
\end{enumerate}
\end{lemma}
\begin{proof}
According to \eqref{eq:dot_z_app}, $\beta$ approaches $\delta^{\text{max}}$ as $t \to \infty$. Thus, demonstrating that $\delta^{\text{max}} \to 0$ as $t \to \infty$ would imply that $\beta\to 0$ as $t \to \infty$.

Note that $\delta^{\text{max}} \in [0, \max(\delta)]$ (condition 1) and $\delta^{\text{max}}$ decreases for all $t$ (condition 2). This indicates that as $t$ goes to infinity, $\delta^{\text{max}}$ must approach a limit $a \in [0, \max(\delta)]$. We will proceed to show that this limit can only be zero. 

Let us define $y^{*}_{0} \in \mathcal{Y}_{0}$ and $t^{*} \in \mathbb{R}_{\geq0}$ as the variables in \eqref{eq:delta_max} where the maxima are achieved for a given $d_{0}$ and $t$. Thus, $\delta^{\text{max}} = \delta(y^{*}_{0}, t^{*})$. From condition 3, we know that ${\delta(y^{*}_{0}, t^{*} + \Delta t)} < \delta(y^{*}_{0}, t^{*})$, $\forall y^{*}_{0} \in \mathcal{L} \setminus \{y_{\mathrm{g}}\}$, $\forall t^{*} \in \mathbb{R}_{\geq 0}$. This leads to two scenarios:
\begin{enumerate}[label=\roman*)]
    \item \emph{$y^{*}_{0}$ remains constant:} If $y^{*}_{0}$ does not change in the interval $[t, t^{*} + \Delta t]$, then $\delta(y^{*}_{0}, t^{*})$, and therefore $\delta^{\text{max}}$, must strictly decrease at some point within this interval. Otherwise, $\delta(y^{*}_{0}, t^{*} + \Delta t) < \delta(y^{*}_{0}, t^{*})$ would not hold true.
    \item \emph{$y^{*}_{0}$ changes:} If $y^{*}_{0}$ changes during the interval $[t, t^{*} + \Delta t]$, given that $\delta^{\text{max}}$ decreases with $t$ (condition 2), this change can only occur if $\delta(y^{*}_{0}, t^{*})$, and hence $\delta^{\text{max}}$, strictly decreases at some point within the interval.
\end{enumerate}

In either case, $\delta^{\text{max}}$ strictly decreases at some point in the interval $[t, t^{*} + \Delta t]$, $\forall d_{0} \in \mathbb{R}_{\geq0} \setminus \{0\}$, $\forall t \in \mathbb{R}_{\geq 0}$. 

Taking this into account, if the limit was some value $a \neq 0$, a contradiction would arise. If $a \neq 0$, we would always have a $d_{0} \neq 0$ such that $\delta^{\text{max}}(d_{0}, 0) = a$ (conditions 4 and 5). This implies that $\delta^{\text{max}}$ must become lower than $a$ before $t^{*} + \Delta t$, and, since it is decreasing (condition 2), it will remain lower as time progresses.  Therefore, the only feasible value for the limit is $a = 0$, which confirms $\delta^{\text{max}} \to 0$ as $t \to \infty$. Consequently, $\beta\to 0$ as $t \to \infty$. 
\end{proof}

\begin{lemma}[$\delta^{\text{max}}$ with $d_{0}$]
\label{lemma:dmax_d0}
$\delta^{\text{max}} \in [0, \max(\delta)]$, and $\delta^{\text{max}}$ surjective with respect to $d_{0}$, if:
\begin{enumerate}[font=\itshape]
\item $\min(\delta^{\text{max}}) = 0$,
\item $\delta^{\text{max}}$ continuous with respect to $d_{0}$.
\end{enumerate}
\end{lemma}
\begin{proof}
Given that $\delta^{\text{max}}$ computes a maximum over values of $\delta$, we get $\max(\delta^{\text{max}}) = \max(\delta)$. Then, we know that  $\min(\delta^{\text{max}}) = 0$ (condition 1), and that $\delta^{\text{max}}$ is continuous with respect to $d_{0}$ (condition 2). Therefore, as a consequence of the \emph{intermediate value theorem}, $\delta^{\text{max}}$ can take any value between $\min(\delta^{\text{max}})$ and $\max(\delta^{\text{max}})$, i.e., $\delta^{\text{max}} \in [0, \max(\delta)]$. Moreover, each value of $\delta^{\text{max}} \in [0, \max(\delta)]$ must have at least one corresponding value of $d_{0}$; hence, in this interval, $\delta^{\text{max}}(d_{0}, t)$ is surjective with respect to $d_{0}$. 
\end{proof}

To summarize, given the conditions and assumptions of Proposition~\ref{prop:existence_KL}, Lemmas \ref{lemma:maxd_decreases}, \ref{lemma:dmax_d0} and Corollary~\ref{coro:d_max_0}, indicate that the conditions of Lemma \ref{lemma:beta_to_zero} are fulfilled, implying that $\beta\to 0$ as $t \to \infty$.

\vspace{0.2cm}
\paragraph{$\beta$ strictly increases with $d_{0}$}
To prove this, we will introduce Lemma \ref{lemma:beta_increases}, which needs a condition supported by Lemma \ref{lemma:z_0_increases}, introduced afterwards.

\begin{lemma}[$\beta$ strictly increases with $d_{0}$]
\label{lemma:beta_increases}
$\beta$ strictly increases with $d_{0}$ if:
\begin{enumerate}[font=\itshape]
\item $z_{0}$ strictly increases with $d_{0}$,
\item $\delta^{\text{max}}(0, t)=0$, $\forall t \in \mathbb{R}_{\geq0}$,
\item $\beta$ strictly decreases with $t$, $\forall d_{0} \in \mathbb{R}_{\geq0}$,
\item $\beta$ continuous with  $t$,
\item $\beta\to 0$ as $t \to \infty$.
\end{enumerate}
\end{lemma}
\begin{proof}
For a fixed $t$, the $\beta$ strictly increases in $d_{0}$ if for any two numbers $d^{\text{a}}_0$ and $d^{\text{b}}_0$ (where $d^{\text{a}}_0 < d^{\text{b}}_0$) within its domain, the condition $d^{\text{a}}_0 < d^{\text{b}}_0$ implies $\beta^{\text{a}}(d^{\text{a}}_0, t) < \beta^{\text{b}}(d^{\text{b}}_0, t)$. 

By condition 1, $z_{0}$ strictly increases with $d_{0}$, so we have $z^{\text{a}}_{0}(d^{\text{a}}_{0}) < z^{\text{b}}_{0}(d^{\text{b}}_{0})$. Moreover, since (any) $z_{0}$ is non-negative (see \eqref{eq:z0_app}), it follows that $z^{\text{b}}_{0} > z^{\text{a}}_{0} \geq 0$. Given condition 2, this indicates that $d^{\text{b}}_{0} > 0$. 

Recalling \eqref{eq:beta_app}, we have 
\begin{equation}
\label{eq:beta_b}
    \beta^{\text{b}} = z^{\text{b}}_{0} + \int_0^{t}\dot{z}_{s}ds.
\end{equation}
Since $d^{\text{b}}_{0} > 0$, it follows from condition 3 that $\beta(d^{\text{b}}_{0}, t)$ strictly decreases with respect to $t$ (i.e., $\dot{z}_t < 0$). Furthermore, $\beta$ is continuous with $t$ and approaches zero as time goes to infinity (conditions 4 and 5). This, combined with $z^{\text{b}}_{0} > z^{\text{a}}_{0}$, implies that there must exist a time $t^{\text{a}}>0$ such that
\begin{equation}
\label{eq:split_int}
    \beta^{\text{b}} = \underbrace{z^{\text{b}}_{0} + \int_{0}^{t^{\text{a}}}\dot{z}_{s}ds}_{z_{0}^{\text{a}}} + \int_{t^{\text{a}}}^{t}\dot{z}_{s}ds.
\end{equation}
Since $\dot{z}_{t} < 0$, starting from the same initial condition, a larger integration interval results in a smaller value of $\beta$. Therefore, we deduce that $\beta^{\text{b}} > \beta^{\text{a}}$, as $\beta^{\text{a}}$ follows the same format as equation \eqref{eq:split_int} (when replacing the terms equivalent to $z_{0}^{\text{a}}$, with $z_{0}^{\text{a}}$) but with an integral that starts at $0$ instead of $t^{\text{a}}$, and ${0 < t^{\text{a}}}$. Thus, under the given conditions, $\beta$ strictly increases with respect to $d_{0}$. 
\end{proof}

\begin{lemma}[$z_{0}$ strictly increases with $d_{0}$]
\label{lemma:z_0_increases}
$z_{0}$ strictly increases with $d_{0}$ if:
\begin{enumerate}[font=\itshape]
\item $\delta(y_{0}, t)$ continuous with respect to $t$. 
\end{enumerate}
\end{lemma}
\begin{proof}
For $z_{0}$ to strictly increase with $d_{0}$, for any two numbers $d^{\text{a}}_0$ and $d^{\text{b}}_0$ within its domain, the inequality $d^{\text{a}}_0 < d^{\text{b}}_0$ must imply $z_{0}^{\text{a}}(d^{\text{a}}_0) < z_{0}^{\text{b}}(d^{\text{b}}_0)$. From \eqref{eq:z0_app}, we have $z_{0} = \delta^{\text{max}}_{0} + d_{0}$. Given that $d_{0}$ strictly increases with itself, it suffices to analyze the behavior of $\delta^{\text{max}}_{0}$. 

Recall that $\delta^{\text{max}}_{0} = \max_{y_{0} \in \mathcal{Y}_{0}}\left(\delta_{t + \Delta t}^{\text{max}}(y_{0}, 0)\right)$ and that \lword{$\delta_{t + \Delta t}^{\text{max}}(y_{0}, 0)$} computes the maximum over the set $\{\delta(y_{0}, s): s \in [0, \Delta t]\}$. We will refer to this set as $\mathcal{W}(y_{0})$. Importantly, $\mathcal{W}(y_{0})$ contains $\delta(y_{0}, 0) = d_{0}$ for all $y_{0}$. Then, for any $d^{\text{a}}_0$ and $d^{\text{b}}_0$ satisfying $d^{\text{a}}_0 < d^{\text{b}}_0$, the behavior of $\delta^{\text{max}}_{0}$ can be studied in three different scenarios. Let $y_{0}^{\text{a}} \in \mathcal{Y}_{0}(d^{\text{a}}_{0})$ be the state that maximizes $\delta^{\text{max}}_{t + \Delta t}$ for $d^{\text{a}}_{0}$, then:

\begin{enumerate}[label=\roman*)]
    \item $\boldsymbol{d^{\text{a}}_{0} =\delta^{\text{max}}_{t + \Delta t}(y^{\text{a}}_{0}, 0)}$: This scenario occurs when $d^{\text{a}}_{0}$ is the maximum of $\mathcal{W}(y^{\text{a}}_{0})$. Then, if $y^{\text{b}}_{0} \in \mathcal{Y}_{0}(d^{\text{b}}_{0})$ is the state that maximizes $\delta^{\text{max}}_{t + \Delta t}$ for $d^{\text{b}}_{0}$, and, hence,  $d^{\text{b}}_0 \in \mathcal{W}(y^{\text{b}}_{0})$, the maximum of $\mathcal{W}(y^{\text{b}}_{0})$ cannot be lower than $d^{\text{b}}_0$. Consequently, since $d^{\text{a}}_0 < d^{\text{b}}_0$, we get $\delta_{t + \Delta t}^{\text{max}}(y^{\text{a}}_{0}, 0) < \delta_{t + \Delta t}^{\text{max}}(y^{\text{b}}_{0}, 0)$. As both $y^{\text{a}}_{0}$ and $y^{\text{b}}_{0}$ are those that maximize $\delta^{\text{max}}_{t + \Delta t}$ over $y_{0}$, for $d^{\text{a}}_{0}$ and $d^{\text{b}}_{0}$, respectively, we conclude that $\delta_{0}^{\text{max}}(d^{\text{a}}_{0}) < \delta_{0}^{\text{max}}(d^{\text{b}}_{0})$.
    
    \item $\boldsymbol{d^{\text{a}}_{0} < \delta^{\text{max}}_{t + \Delta t}(y^{\text{a}}_{0}, 0) \leq d^{\text{b}}_{0}}$: This scenario occurs when the maximum of $\mathcal{W}(y^{\text{a}}_{0})$ is not $d^{\text{a}}_{0}$, and $d^{\text{b}}_{0}$ is greater than or equal to this maximum. Following a similar reasoning to the above, we can conclude that $\delta_{0}^{\text{max}}(d^{\text{a}}_{0}) \leq \delta_{0}^{\text{max}}(d^{\text{b}}_{0})$.

    \item $\boldsymbol{d^{\text{a}}_{0} < d^{\text{b}}_{0} < \delta^{\text{max}}_{t + \Delta t}(y^{\text{a}}_{0}, 0)}$: This scenario occurs when the maximum of $\mathcal{W}(y^{\text{a}}_{0})$ is not $d^{\text{a}}_{0}$, and $d^{\text{b}}_{0}$ is lower than this maximum. Provided that $\delta$ is continuous with $t$ (condition 1), in the set $\mathcal{W}(y^{\text{a}}_{0})$, since, for $t=0$, $\delta=d^{\text{a}}_{0}$, and for some $t^{*} > 0$, ${\delta = \delta^{\text{max}}_{t + \Delta t}(y^{\text{a}}_{0}, 0)}$, then we know there must exist a $t^{\text{b}}$, with ${0 < t^{\text{b}} < t^{*} \leq \Delta t}$, where $\delta=d^{\text{b}}_{0}$. 
    
    Now, observe that $t^{*} \in [t^{\text{b}}, t^{\text{b}} + \Delta t]$. Therefore, $\delta^{\text{max}}_{t + \Delta t}(y^{\text{a}}_{0},$ $t^{\text{b}})$, which computes the maximum over $\{\delta(y^{\text{a}}_{0}, s): s \in [t^{\text{b}}, t^{\text{b}} + \Delta t]\}$, cannot be lower than $\delta^{\text{max}}_{t + \Delta t}(y^{\text{a}}_{0}, 0)$. Moreover, we can select $y^{\text{a}}_{t}$ at time $t^{\text{b}}$, i.e., $y^{\text{a}}_{t^{\text{b}}}$, as a new initial condition such that $\delta^{\text{max}}_{t + \Delta t}(y^{\text{a}}_{t^{\text{b}}}, 0) = \delta^{\text{max}}_{t + \Delta t}(y^{\text{a}}_{0}, t^{\text{b}})$. Since $\delta(y^{\text{a}}_{t^{\text{b}}}, 0) = d_{0}^{\text{b}}$, we get that $y^{\text{a}}_{t^{\text{b}}} \in \mathcal{Y}_{0}(d^{\text{b}}_{0})$. Hence, even though $y^{\text{a}}_{t^{\text{b}}}$ might not maximize $\delta^{\text{max}}_{t + \Delta t}$ for $\mathcal{Y}_{0}(d^{\text{b}}_{0})$, we know that the maximum, achived at $y^{\text{b}}_{0}$, must be at least greater than or equal to $\delta^{\text{max}}_{t + \Delta t}(y^{\text{a}}_{t^{\text{b}}}, 0)$. 
    
    Consequently, we have that $\delta^{\text{max}}_{t + \Delta t}(y^{\text{b}}_{0}, 0) \geq \delta^{\text{max}}_{t + \Delta t}(y^{\text{a}}_{t^{\text{b}}}, 0)$, and $\delta^{\text{max}}_{t + \Delta t}(y^{\text{a}}_{t^{\text{b}}}, 0) = \delta^{\text{max}}_{t + \Delta t}(y^{\text{a}}_{0}, t^{\text{b}}) \geq \delta^{\text{max}}_{t + \Delta t}(y^{\text{a}}_{0}, 0)$. Hence, ${\delta^{\text{max}}_{t + \Delta t}(y^{\text{a}}_{0}, 0) \leq \delta^{\text{max}}_{t + \Delta t}(y^{\text{b}}_{0}, 0)}$, and, therefore, $\delta_{0}^{\text{max}}(d^{\text{a}}_{0}) \leq \delta_{0}^{\text{max}}(d^{\text{b}}_{0})$.
\end{enumerate}

In summary, for any $d^{\text{a}}_{0}$ and $d^{\text{b}}_{0}$ with $d^{\text{a}}_{0} < d^{\text{b}}_{0}$, it is always true that $\delta_{0}^{\text{max}}(d^{\text{a}}_{0}) \leq \delta_{0}^{\text{max}}(d^{\text{b}}_{0})$, meaning $\delta^{\text{max}}_{0}$ is non-decreasing with respect to $d_{0}$. As the sum of a non-decreasing function ($\delta^{\text{max}}_{0}$) and a strictly increasing function ($d_{0}$) results in a strictly increasing function, we conclude that $z_{0}$ strictly increases with $d_{0}$.
\end{proof}

From Lemma \ref{lemma:beta_increases}, we can conclude that $\beta$ strictly increases with respect to $d_{0}$, since assuming the conditions of Proposition~\ref{prop:existence_KL}, and that we have proven that $\beta$ is continuous in $t$, Lemmas \ref{lemma:z_0_increases}, \ref{lemma:upper_decreasing}, and \ref{lemma:beta_to_zero} indicate that the conditions of Lemma \ref{lemma:beta_increases} are fulfilled.

\vspace{0.3cm}
\noindent\fbox{%
    \parbox{\columnwidth}{%
       To summarize, we have demonstrated that all the requirements for $\beta(d_{0}, t)$ to be a class-$\mathcal{KL}$ function, and to serve as an upper bound for $\delta(y_{0}, t)$, are met for every $y_{t} \in \mathcal{L}$ and for all $t \in \mathbb{R}_{\geq 0}$. With this, our proof of Proposition~\ref{prop:existence_KL} is complete. \hfill\qedsymbol
 
    }%
}
\renewcommand{\qedsymbol}{} 
\end{proof}
\renewcommand{\qedsymbol}{\openbox} 

%% file: sections/appendix/B.tex
\section{Learning on spherical manifolds}
\label{learning_spheres}

In this work, we employ unit quaternions to control orientation; therefore, we consider two distance functions that are relevant when learning spherical geometries: 1) great-circle distance, and 2) chordal distance.

\begin{itemize}
    \item \textbf{Great-circle distance:} This is the distance along a \emph{great circle}. A great circle is the largest circle that can be drawn on any given sphere, defining the shortest distance between two points. In the context of unit quaternions, which have a unitary norm, this distance is equivalent to the central angle $\alpha$ subtended by two points on the sphere. Hence, we can define it as
    \begin{equation}
        d_{\text{g.c.}} = \alpha.
    \end{equation}
    \item \textbf{Chordal distance:} This is a distance that can be computed when an n-sphere is embedded in a higher-dimensional Euclidean space, i.e., $\mathcal{S}^{n} \subset \mathbb{R}^{n+1}$. Then, by computing the Euclidean distance in $\mathbb{R}^{n+1}$ between two points in $\mathcal{S}^{n}$, we \emph{induce} a distance in $\mathcal{S}^{n}$, corresponding to the chordal distance \citep{lee2006riemannian,berg2008complex,jeong2017spherical}. As the name suggests, this distance is the length of the chord connecting two points in an n-sphere. Hence, it is defined as 
    \begin{equation}
        \label{eq:d_chord}
        d_{\text{chord}}=2r\sin(\alpha/2),
    \end{equation}
    where $r$ is the radius of the n-sphere. 
\end{itemize}
These distances are depicted in Fig. \ref{fig:chords}. Since both define the same topology $\mathcal{S}^{n}$ they are considered to be equivalent, and can be utilized in PUMA to enforce stability when states are represented as unit quaternions. The great circle distance is especially suited to spherical spaces as it defines the \emph{shortest path}, or the \emph{geodesic}, between two points. Conversely, the chordal distance is straightforward to apply because it naturally arises when calculating the Euclidean distance at the output of $\psi_{\theta}$. This is due to $\mathcal{L} \subset \mathbb{R}^{m}$, where $m>n$ represents the output size of $\psi_{\theta}$.

\subsection{Comments on local stability}
Considering a one-dimensional sphere, it is notable that at $\alpha = \pi$ (with respect to the goal), both the great-circle distance and the chordal distance can decrease in two possible ways: by evolving either to the right or to the left (see Fig. \ref{fig:chords}). Consequently, to ensure these distances decrease in the region around $\alpha = \pi$, one of these options should be selected. However, this would require the dynamical system to instantly change its direction at this point, rendering $ f^{\mathcal{T}}_{\theta} $ discontinuous. Yet, in Theorem~\ref{theo:puma_surrogate}, we assume this function is continuous, as continuity is a prerequisite for the class-$ \mathcal{KL} $ upper bound $ \beta $ to also be continuous, which is necessary to prove stability.

Therefore, by extending this idea to higher-dimensional spheres, this implies that when employing these metrics, we can only enforce local asymptotic stability for $\mathcal{S}^{n} \setminus \{ p \}$, where $p$ is the point at $\alpha = \pi$. Moreover, due to the continuity of $f^{\mathcal{T}}_{\theta}$, this point must correspond to a zero, and, therefore, represents an unstable equilibrium. This property is not a flaw in the great-circle distance or the chordal distance. Rather, it is a result of the topology of $\mathcal{S}^{n}$, which does not allow for the existence of a single stable equilibrium. This is a consequence of the Poincar\'{e}-Hopf theorem \citep{guillemin2010differential}.
\begin{figure}[t]
    \centering
    \includegraphics[width=0.8\columnwidth]{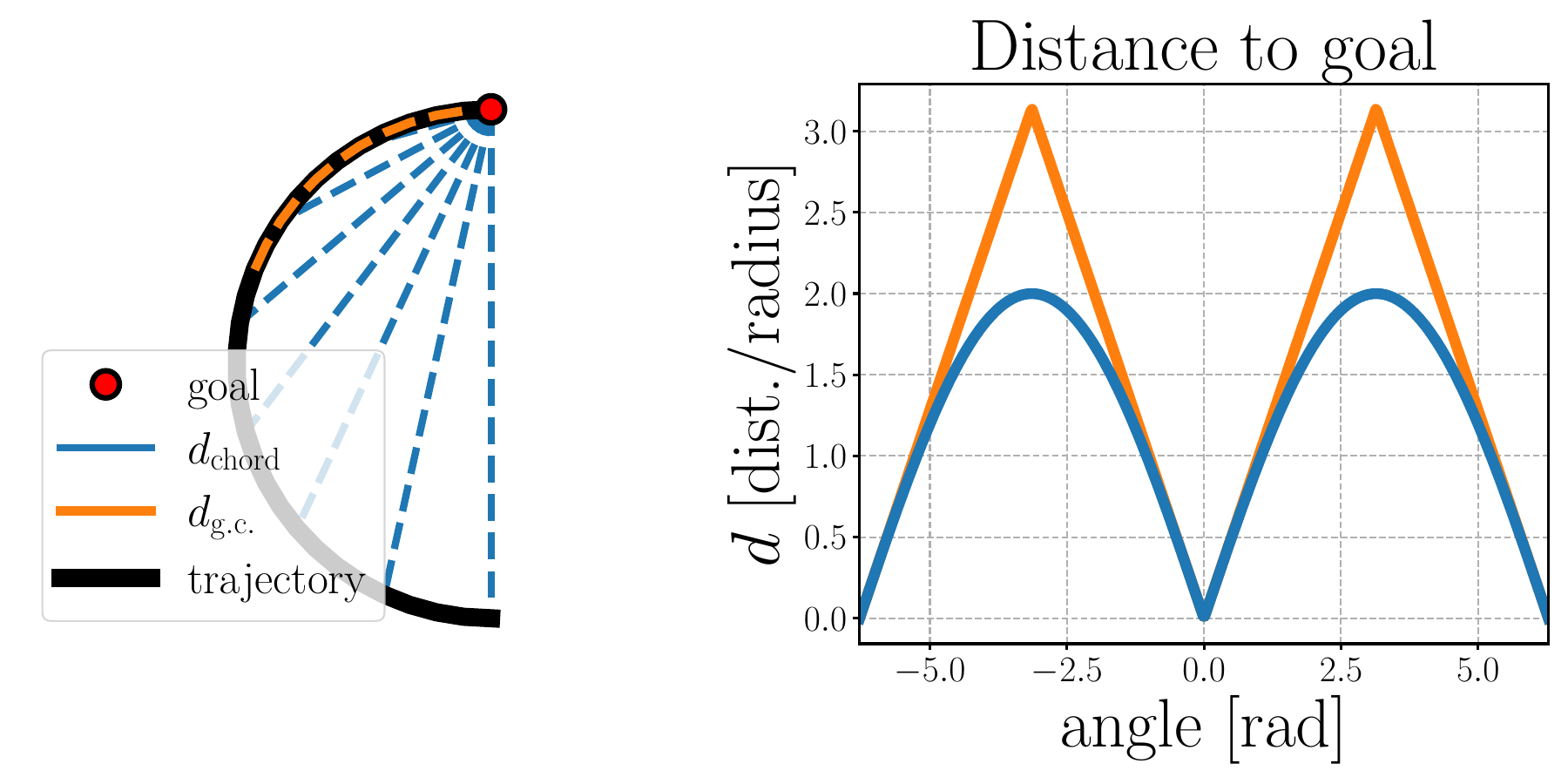}
    \caption{Left: A spherical trajectory towards a goal at the north pole. The chordal and great-circle distances at specific points are indicated as $d_{\text{chord}}$ and $d_{\text{g.c.}}$, respectively. Right: Distances as a function of the central angle between points.}
    \label{fig:chords}
\end{figure}
\input{sections/appendix/hyperparam_table}
\subsection{Pose Control}
Until now, our discussion has centered on the applicability of our method for unit quaternions. However, practical control of a robot's pose requires simultaneous control of both its position (Euclidean) and orientation (non-Euclidean). This can be achieved by using a \emph{product metric}, i.e., a metric resulting from the Cartesian product of spaces. In this case, we are looking at the product $\mathbb{R}^{3} \times \mathcal{S}^{3}$.

A simple product metric is the sum of the metrics from each space in the product \citep{deza2009encyclopedia}, e.g., $d_{\mathbb{R}^{n}} + d_{\mathcal{S}^{n}}$ for $\mathbb{R}^{n} \times \mathcal{S}^{n}$, where $d_{\mathbb{R}^{n}}$ is a distance in $\mathbb{R}^{n}$ and $d_{\mathcal{S}^{n}}$ is a distance in $\mathcal{S}^{n}$.  However, it is not straightforward to do this in $\mathcal{L}$ without modifying the DNN structure. This is because the latent states $y_{t}$ are an entangled representation of the robot states $x_{t}$, making it impossible to simply add together the distances of the Euclidean and non-Euclidean parts in the latent space.

Interestingly, when the product metric is computed using manifolds of identical topology, such as $\mathbb{R} \times \mathbb{R}$, the resulting metric is \emph{equivalent} to the one obtained by directly computing the metric in the higher-dimensional space ($\mathbb{R}^{2}$ in this example) \citep{deza2009encyclopedia}. Hence, for such scenarios, there is no need to explicitly disentangle the states in $\mathcal{L}$; instead, we can compute the metric directly in the complete latent space. Our case, nevertheless, is more complex since the topologies of $\mathcal{S}^{3}$ and $\mathbb{R}^{3}$ are different. Fortunately, we found two ways of easily overcoming this limitation.

\subsubsection{Pose in Euclidean space}
We have observed that an Euclidean metric, the chordal distance, can generate spherical metrics in lower-dimensional manifolds. This becomes particularly relevant when a Euclidean metric is employed in $\mathbb{R}^{m}$, where $m$ exceeds the dimensionality of our state space (6 in our case). In this scenario, the metric is equivalent to $d_{\mathbb{R}^{3}} + d_{\mathbb{R}^{m'}}$ for $m' > 3$ and $3 + m'=m$. Given that $\mathcal{S}^{3}$ can be induced inside $\mathbb{R}^{m'}$, this suggests that $\mathbb{R}^{3} \times \mathcal{S}^{3}$ can be induced in $\mathbb{R}^{m}$ when using the Euclidean distance in this space.

\subsubsection{Pose in spherical space}
Finally, we also note that the great-circle distance can be employed to achieve the same objective. Suppose we compute this distance in $\mathcal{S}^{6}$. Then, it would be equivalent to $d_{\mathcal{S}^{3}} + d_{\mathcal{S}^{3}}$. Interestingly, a diffeomorphism can be found between $\mathbb{R}^{n}$ and a subset of $\mathcal{S}^{n}$, e.g., the \emph{stereographic projection} \citep{oprea2007differential}. This implies that it is feasible to use the metric $d_{\mathcal{S}^{3}}$ in $\mathcal{L}$, and find a valid representation for $\mathbb{R}^{3}$ within a subspace of $\mathcal{S}^{3}$. Consequently, the product space $\mathbb{R}^{3} \times \mathcal{S}^{3}$ can be represented within a subset of $\mathcal{S}^{6}$.

%% file: sections/appendix/hyperparam_table.tex
\begin{table*}[t]
\scriptsize
\centering
\caption{Hyperparameter optimization results of the different variations of PUMA.}
\label{tab:hyperparams}
\resizebox{\textwidth}{!}{
\begin{tabular}{lclccccccccc}
\hline
\textbf{Hyperparameter}                        & \textbf{Opt.?}                &  & \multicolumn{4}{c}{\textbf{\begin{tabular}[c]{@{}c@{}}Hand-tuned value / \\ initial opt. guess\end{tabular}}}                                                                           & \textbf{}                     & \multicolumn{4}{c}{\textbf{Optimized value}}                                                                                                                                            \\ \cline{4-7} \cline{9-12} 
\textbf{}                                      & \multicolumn{1}{l}{\textbf{}} &  & \textbf{Euc.}        & \textbf{Sph.}        & \textbf{\begin{tabular}[c]{@{}c@{}}Euc.\\ (2nd order)\end{tabular}} & \textbf{\begin{tabular}[c]{@{}c@{}}Sph.\\ (2nd order)\end{tabular}} & \multicolumn{1}{l}{\textbf{}} & \textbf{Euc.}        & \textbf{Sph.}        & \textbf{\begin{tabular}[c]{@{}c@{}}Euc.\\ (2nd order)\end{tabular}} & \textbf{\begin{tabular}[c]{@{}c@{}}Sph.\\ (2nd order)\end{tabular}} \\
\textbf{CONDOR}                                & \multicolumn{1}{l}{}          &  & \multicolumn{1}{l}{} & \multicolumn{1}{l}{} & \multicolumn{1}{l}{}                                                & \multicolumn{1}{l}{}                                                & \multicolumn{1}{l}{}          & \multicolumn{1}{l}{} & \multicolumn{1}{l}{} & \multicolumn{1}{l}{}                                                & \multicolumn{1}{l}{}                                                \\ \hline
Stability loss margin ($m$)                    & \cmark                        &  & 1.250e-8             & 1.250e-4             & 1.250e-4                                                            & 1.250e-4                                                            &                               & 5.921e-3             & 3.012e-05            & 2.424e-08                                                           & 2.919e-7                                                            \\
Triplet imitation loss weight ($\lambda$)      & \cmark                        &  & 1                    & 1                    & 1                                                                   & 1                                                                   &                               & 1.315e-1             & 3.496                & 1.022e-1                                                            & 4.473e-1                                                            \\
Window size imitation ($\mathcal{H}^{i}$)      & \cmark                        &  & 14                   & 14                   & 14                                                                  & 14                                                                  &                               & 13                   & 13                   & 14                                                                  & 14                                                                  \\
Window size stability ($\mathcal{H}^{s}$)      & \cmark                        &  & 4                    & 1                    & 1                                                                   & 1                                                                   &                               & 11                   & 13                   & 11                                                                  & 11                                                                  \\
Batch size imitation ($\mathcal{B}^{i}$)       & \xmark                        &  & 250                  & 250                  & 250                                                                 & 250                                                                 &                               & -                    & -                    & -                                                                   & -                                                                   \\
Batch size stability ($\mathcal{B}^{s}$)       & \xmark                        &  & 250                  & 250                  & 250                                                                 & 250                                                                 &                               & -                    & -                    & -                                                                   & -                                                                   \\
                                               & \multicolumn{1}{l}{}          &  & \multicolumn{1}{l}{} & \multicolumn{1}{l}{} & \multicolumn{1}{l}{}                                                & \multicolumn{1}{l}{}                                                & \multicolumn{1}{l}{}          & \multicolumn{1}{l}{} & \multicolumn{1}{l}{} & \multicolumn{1}{l}{}                                                & \multicolumn{1}{l}{}                                                \\
\textbf{Neural Network}                        & \multicolumn{1}{l}{}          &  & \multicolumn{1}{l}{} & \multicolumn{1}{l}{} & \multicolumn{1}{l}{}                                                & \multicolumn{1}{l}{}                                                & \multicolumn{1}{l}{}          & \multicolumn{1}{l}{} & \multicolumn{1}{l}{} & \multicolumn{1}{l}{}                                                & \multicolumn{1}{l}{}                                                \\ \hline
Optimizer                                      & \xmark                        &  & Adam                 & Adam                 & Adam                                                                & Adam                                                                &                               & -                    & -                    & -                                                                   & -                                                                   \\
Number of iterations                           & \xmark                        &  & 40000                & 40000                & 40000                                                               & 40000                                                               &                               & -                    & -                    & -                                                                   & -                                                                   \\
Learning rate                                  & \cmark                        &  & 1e-4                 & 1e-4                 & 1e-4                                                                & 1e-4                                                                &                               & 9.784e-5             & 8.574e-4             & 1.670e-4                                                            & 1.245e-4                                                            \\
Activation function                            & \xmark                        &  & GELU                 & GELU                 & GELU                                                                & GELU                                                                &                               & -                    & -                    & -                                                                   & -                                                                   \\
Num. layers ($\psi_{\theta}$, $\phi_{\theta}$) & \xmark                        &  & (3, 3)               & (3, 3)               & (3, 3)                                                              & (3, 3)                                                              &                               & -                    & -                    & -                                                                   & -                                                                   \\
Neurons/hidden layer                           & \xmark                        &  & 300                  & 300                  & 300                                                                 & 300                                                                 &                               & -                    & -                    & -                                                                   & -                                                                   \\
Layer normalization                            & \xmark                        &  & yes                  & yes                  & yes                                                                 & yes                                                                 &                               & -                    & -                    & -                                                                   & -                                                                   \\ \hline
\end{tabular}
}
\end{table*}

%% file: sections/appendix/C.tex
\section{Hyperparameter Optimization}
\label{appendix:hyperparam}
To optimize the hyperparameters of the different variations of PUMA employed in the LASA and LAIR datasets, we utilized the Tree Parzen Estimator \citep{bergstra2011algorithms}. This optimization method builds a probability model that facilitates the selection of the most promising set of hyperparameters in each optimization round. For the implementation of the Tree Parzen Estimator, we used the Optuna API \citep{akiba2019optuna}. 

To evaluate the selected sets of hyperparameters, we employ a loss function $\mathcal{L}_{\text{hyper}}$ composed of two components. The first component, $\mathcal{L}_{\text{acc}}$, assesses the accuracy of the trained model. This is done by computing the RMSE between the demonstrations and the trajectories simulated by the learned model, in a manner similar to the approach used in the experiments section. The second component, $\mathcal{L}_{\text{goal}}$, quantifies the average distance between the final points of trajectories, simulated using the learned model, and the target goal. Thus, we have:
\begin{equation}
\mathcal{L}_{\text{hyper}} = \mathcal{L}_{\text{acc}} + \gamma \mathcal{L}_{\text{goal}}
\end{equation}
where $\gamma$ is a weighting factor. After initial tests, we settled on $\gamma=3.5$.

To make the computationally intensive process of optimizing hyperparameters more feasible, we employed six strategies: 1) reducing the overhead of the objective function, 2) limiting the size of the evaluation set, 3) utilizing Bayesian optimization, 4) applying pruning techniques, 5) selecting a subset of hyperparameters for optimization, and 6) employing an optimization range. For details on the first five strategies, the reader is referred to \citet{perez2023stable}. The sixth strategy consists of restricting the hyperparameter search to a predefined range.

Table \ref{tab:hyperparams} presents the results of this optimization process. We can observe the hyperparameters that were optimized, their initial values (chosen based on preliminary tests), and their final optimized values. The ranges for hyperparameter optimization are as follows: 1) $m: [1e-9, 1e-1]$, 2) $\lambda: [1e-1, 10]$, 3) $\mathcal{H}^{i}: [1, 14]$, 4) $\mathcal{H}^{s}: [1, 14]$, and 5) learning rate: $[1e-5, 1e-3]$. Note that the variations using the boundary loss are not included in the table, since in those cases the same hyperparameters were employed and the boundary loss was added with a weight of $0.001$. The same holds for the behavioral cloning case. Lastly, regarding Sec. \ref{sec:boundary_non_euc} it is important to note that in every experiment involving spherical state spaces, the dynamical system was kept within the manifold by normalizing the forward Euler integration output.